\documentclass{article} 
\usepackage{iclr2021_conference,times}

\usepackage{amsmath,amsfonts,bm}

\def\eqref#1{equation~\ref{#1}}

\def\1{\bm{1}}

\DeclareMathAlphabet{\mathsfit}{\encodingdefault}{\sfdefault}{m}{sl}
\SetMathAlphabet{\mathsfit}{bold}{\encodingdefault}{\sfdefault}{bx}{n}

\usepackage[utf8]{inputenc} 
\usepackage[T1]{fontenc}    
\usepackage{hyperref}       
\usepackage{url}            
\usepackage{booktabs}       
\usepackage{amsfonts}       
\usepackage{nicefrac}       
\usepackage{microtype}      
\usepackage{amsmath}
\usepackage{graphicx}
\usepackage{caption}

\usepackage{amsthm}
\newtheorem{theorem}{Theorem}
\newcommand{\norms}[1]{\Vert#1\Vert}
\newcommand{\iprods}[1]{\langle #1\rangle}

\usepackage{algorithmicx}
\usepackage[ruled]{algorithm}
\usepackage[noend]{algpseudocode}
\usepackage{multirow}
\usepackage[normalem]{ulem}
\useunder{\uline}{\ul}{}
\usepackage{adjustbox}
\usepackage{enumitem}
\usepackage{forloop}
\usepackage{pgffor}
\usepackage{authblk}

\mathchardef\mhyphen="2D

\newcommand{\RandConv}{\texttt{RandConv}}

\usepackage{todonotes}

\title{Robust and Generalizable Visual Representation Learning via Random Convolutions}

\author[1]{\textbf{Zhenlin Xu}}
\author[1]{\textbf{Deyi Liu}}
\author[2]{\textbf{Junlin Yang}}
\author[1]{\textbf{Colin Raffel}}
\author[1]{\textbf{Marc Niethammer}}
\affil[1]{
University of North Carolina at Chapel Hill }
\affil[2]{
Yale University}
\affil[1]{\footnotesize\texttt{\{zhenlinx, mn, craffel\}@cs.unc.edu, deyi@live.unc.edu}}
\affil[2]{\footnotesize\texttt{junlin.yang@yale.edu}}

\iclrfinalcopy 
\begin{document}

\maketitle

\begin{abstract}
	While successful for various computer vision tasks, deep neural networks have shown to be vulnerable to texture style shifts and small perturbations to which humans are robust. In this work, we show that the robustness of neural networks can be greatly improved through the use of random convolutions as data augmentation. Random convolutions are approximately shape-preserving and may distort local textures. Intuitively, randomized convolutions create an infinite number of new domains with similar global shapes but random local texture. Therefore, we explore using outputs of multi-scale random convolutions as new images or mixing them with the original images during training. When applying a network trained with our approach to unseen domains, our method consistently improves the performance on domain generalization benchmarks and is scalable to ImageNet. In particular, in the challenging scenario of generalizing to the sketch domain in PACS and to ImageNet-Sketch, our method outperforms state-of-art methods by a large margin. More interestingly, our method can benefit downstream tasks by providing a more robust pretrained visual representation. \footnote{Code is available at \url{ https://github.com/wildphoton/RandConv}.}
\end{abstract}

\section{Introduction}
\label{Introduction}

Generalizability and robustness to out-of-distribution samples have been major pain points when applying deep neural networks (DNNs) in real world applications~\citep{volpi2018generalizing}. Though DNNs are typically trained on datasets with millions of training samples, they still lack robustness to domain shift, small perturbations, and adversarial examples~\citep{luo2019taking}.  Recent research has shown that neural networks tend to use superficial features rather than global shape information for prediction even when trained on large-scale datasets such as ImageNet~\citep{geirhos2018imagenettrained}. These superficial features can be local textures or even patterns imperceptible to humans but detectable to DNNs, as is the case for adversarial examples~\citep{ilyas2019adversarial}. In contrast, image semantics often depend more on object shapes rather than local textures. For image data, local texture differences are one of the main sources of domain shift, e.g., between synthetic virtual images and real data~\citep{sun2014virtual}. Our goal is therefore to learn visual representations that are invariant to local texture and that generalize to unseen domains. {While texture and color may be treated as different concepts, we follow the convention in \cite{geirhos2018imagenettrained} and include color when talking about texture.}

We address the challenging setting of robust visual representation learning from \emph{single domain data}. Limited work exists in this setting. Proposed methods include data augmentation~\citep{volpi2018generalizing, qiao2020learning, geirhos2018imagenettrained}, domain randomization~\citep{tobin2017domain,yue2019domain}, self-supervised learning~\citep{carlucci2019jigen}, and penalizing the predictive power of low-level network features~\citep{wang2019learning}. Following the spirit of adding inductive bias towards global shape information over local textures, we propose using random convolutions to improve the robustness to domain shifts and small perturbations. While recently \citet{lee2020network} proposed a similar technique for improving the generalization of reinforcement learning agents in unseen environments, we focus on visual representation learning and examine our approach on visual domain generalization benchmarks. Our method also includes the multiscale design and a mixing variant.
In addition, considering that many computer vision tasks rely on training deep networks based on ImageNet-pretrained weights (including some domain generalization benchmarks), we ask \emph{``Can a more robust pretrained model make the finetuned model more robust on downstream tasks?''} Different from~\citep{kornblith2019better,salman2020adversarially} who studied the transferability of a pretrained ImageNet representation to new tasks while focusing on in-domain generalization, we explore generalization performance on \emph{unseen domains} for new tasks.

\begin{figure}[t]
	\begin{center}
		\renewcommand{\arraystretch}{0.5}
		\setlength{\tabcolsep}{0.00cm}
		\newcommand\cwidth{0.14\textwidth}
		\begin{adjustbox}{max width=\textwidth}

			\begin{tabular}{ccccccc}
    			\multicolumn{7}{c}{\includegraphics[width=0.9\textwidth]{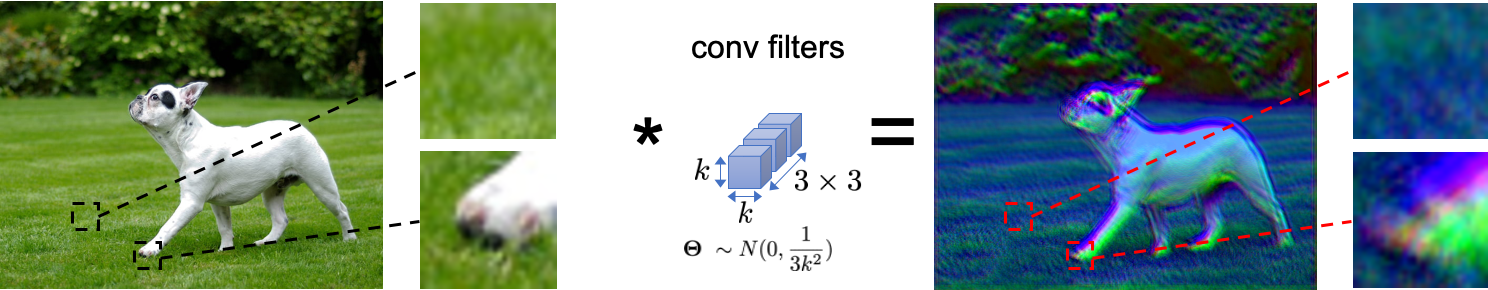}} \\\\
				\newcounter{imgnum}
				\newcounter{sample_id}
				\newcounter{ks}
				Input & $k=1$ & $k=3$ & $k=5$ & $k=7$ & $k=11$ & $k=15$\\
				\forloop{imgnum}{2}{\value{imgnum} < 3}{
					\includegraphics[width=\cwidth]{Fig/examples/image\arabic{imgnum}.png} 
					\forloop{sample_id}{2}{\value{sample_id} < 3}{
						& \includegraphics[width=\cwidth]{Fig/examples/image\arabic{imgnum}_kernel1_sample\arabic{sample_id}.png}
						& \includegraphics[width=\cwidth]{Fig/examples/image\arabic{imgnum}_kernel3_sample\arabic{sample_id}.png} 
						& \includegraphics[width=\cwidth]{Fig/examples/image\arabic{imgnum}_kernel5_sample\arabic{sample_id}.png}
						& \includegraphics[width=\cwidth]{Fig/examples/image\arabic{imgnum}_kernel7_sample\arabic{sample_id}.png} 
						& \includegraphics[width=\cwidth]{Fig/examples/image\arabic{imgnum}_kernel11_sample\arabic{sample_id}.png}
						& \includegraphics[width=\cwidth]{Fig/examples/image\arabic{imgnum}_kernel15_sample\arabic{sample_id}.png}
						\\
					}  
				}\\[-2mm]
				Input & $\alpha=0.9$ & $\alpha=0.7$ & $\alpha=0.5$ & $\alpha=0.3$ & $\alpha=0.1$ & $\alpha=0$ \\
				\includegraphics[width=\cwidth]{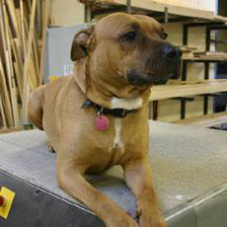} 
				&\includegraphics[width=\cwidth]{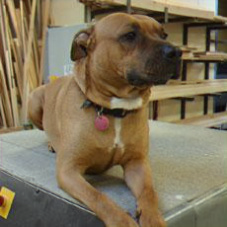} 
				&\includegraphics[width=\cwidth]{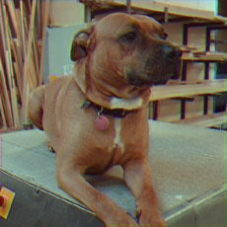} 
				&\includegraphics[width=\cwidth]{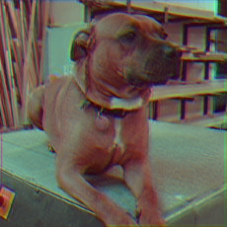} 
				&\includegraphics[width=\cwidth]{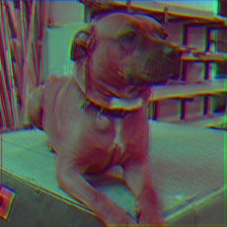} 
				&\includegraphics[width=\cwidth]{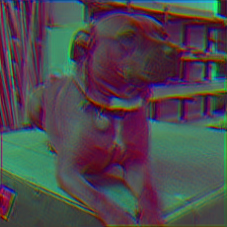} 
				&\includegraphics[width=\cwidth]{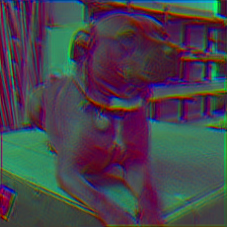} 
			\end{tabular}
		\end{adjustbox}
	\end{center}
	\vspace{-3mm}
	\caption{\small \textbf{Top}: Illustration that {\RandConv} randomize local texture but preserve shapes in the image. \textbf{Middle:} First column is the input image of size $224^2$; following columns are convolutions results using random filters of different sizes $k$. \textbf{Bottom:} Mixing results between an image and one of its random convolution results with different mixing coefficients $\alpha$.}  
	\vspace{-5mm}
	\label{fig:randconv_example}
\end{figure}

We make the following contributions:
\begin{itemize}[leftmargin=2em]
	\vspace{-0mm}
	\setlength\itemsep{0em}
	\item We develop {\RandConv}, a data augmentation technique \emph{using multi-scale random-convolutions to generate images with random texture while maintaining global shapes.} We explore using the {\RandConv} output as training images or mixing it with the original images. We show that a consistency loss can further enforce invariance under texture changes.
	\item We provide insights and justification on why {\RandConv} augments images with different local texture but the same semantics with the shape-preserving property of random convolutions.
	\item We validate {\RandConv} and its mixing variant in extensive experiments on synthetic and real-world benchmarks as well as on the large-scale ImageNet dataset. Our methods outperform single domain generalization approaches by a large margin on digit recognition datasets and for the challenging case of generalizing to the Sketch domain in PACS and to ImageNet-Sketch.
	\item  We explore if the robustness/generalizability of a pretrained representation can transfer. We show that transferring a model pretrained with {\RandConv} on ImageNet can further improve domain generalization performance on new downstream tasks on the PACS dataset.
\end{itemize}

\section{Related Work}
\textbf{Domain Generalization} (DG) aims at learning representations that perform well when transferred to unseen domains. Modern techniques range between feature fusion~\citep{shen2019situational}, meta-learning~\citep{li2018mldg, balaji2018metareg}, and adversarial training~\citep{shao2019multi, li2018domain}. Note that most current DG work~\citep{ghifary2016scatter, li2018mldg, li2018domain} requires a multi-source training setting to work well. However, in practice, it might be difficult and expensive to collect data from multiple sources, such as collecting data from multiple medical centers~\citep{raghupathi2014big}. 
Instead, we consider the more strict single-domain generalization DG setting, where we train the model on source data from a single domain and generalize it to new unseen domains~\citep{carlucci2019jigen, wang2018learning}.

\textbf{Domain Randomization} (DR) was first introduced as a DG technique by~\citet{tobin2017domain} to handle the domain gap between simulated and real data. As the training data in~\citep{tobin2017domain} is synthesized in a virtual environment, it is possible to generate diverse training samples by randomly selecting background images, colors, lighting, and textures of foreground objects. When a simulation environment is not accessible, image stylization can be used to generate new domains~\citep{yue2019domain,geirhos2018imagenettrained}. However, this requires extra effort to collect data and to train an additional model; further, the number of randomized domains is limited by the number of predefined styles.

\textbf{Data Augmentation} has been widely used to improve the generalization of machine learning models~\citep{simard2003best}. DR approaches can be considered a type of synthetic data augmentation. To improve performance on unseen domains, \citet{volpi2018generalizing} generate adversarial examples to augment the training data;  \citet{qiao2020learning} extend this approach via meta-learning. As with other adversarial training 
algorithms, significant extra computation is required to obtain adversarial examples.

\textbf{Learning Representations Biased towards Global Shape}
\citet{geirhos2018imagenettrained} demonstrated that convolutional neural networks (CNNs) tend to use superficial local features even when trained on large datasets. To counteract this effect, they proposed to train on stylized ImageNet, thereby forcing a network to rely on object shape instead of textures. Wang et al. improved out-of-domain performance by penalizing the correlation between a learned representation and superficial features such as the gray-level co-occurrence matrix~\citep{wang2018learning}, or by penalizing the predictive power of local, low-level layer features in a neural network via an adversarial classifier~\citep{wang2019learning}. Our approach shares the idea that learning representations invariant to local texture helps generalization to unseen domains. However, {\RandConv} avoids searching over many hyper-parameters, collecting extra data, and training other networks. It also scales to large-scale datasets since it adds minimal computation overhead.

\textbf{Random Mapping in Machine Learning}
Random projections have also been effective for dimensionality reduction based on the distance-preserving property of the Johnson–Lindenstrauss lemma~\citep{johnson1984extensions}. \citep{vinh2016training} applied random projections on entire images as data augmentation to make neural networks robust to adversarial examples. \citet{lee2020network} recently used random convolutions to help reinforcement learning (RL) agents generalize to new environments. Neural networks with \textit{fixed} random weights can encode meaningful representations \citep{saxe2011random} and are therefore useful for neural architecture search~\citep{gaier2019weight}, generative models~\citep{he2016powerful}, natural language processing~\citep{wieting2018no}, and RL~\citep{osband2018randomized, burda2018exploration}. In contrast, {\RandConv} uses \textit{non-fixed} randomly-sampled weights to generate images with different local texture.

\section{RandConv: Randomize Local Texture at Different Scales }

We propose using a convolution layer with non-fixed random weights as the first layer of a DNN during training. This strategy generates images with random local texture but consistent shapes, and is beneficial for robust visual representation learning. 
Sec.~\ref{sec:rc_shape_preserve} justifies the shape-preserving property of a random convolution layer. 
Sec.~\ref{sec:randconv} describes {\RandConv}, our data augmentation algorithm using a multi-scale randomized convolution layer and input mixing.

\subsection{A Random Convolution Layer Preserves Global Shapes}
\label{sec:rc_shape_preserve}
Convolution is the key building block for deep convolutional neural networks.
Consider a convolution layer with filters $\mathbf{\Theta}\in\mathbb{R}^{h\times w \times C_{in}\times C_{out}}$ with an input image $\mathbf{I}\in\mathbb{R}^{H\times W\times C_{in}}$, where $H$ and $W$ are the height and width of the input and $C_{in}$ and $C_{out}$ are the number of feature channels for the input and output, and $h$ and $w$ are the height and width of the layer's filter. The output (with appropriate input padding) will be $\mathbf{g} = \mathbf{I} *  \mathbf{\Theta}$ with $\mathbf{g}\in\mathbb{R}^{H\times W\times C_{out}}$.  

 In images, nearby pixels with similar color or texture can be grouped into primitive shapes that represent parts of objects or the background. A convolution layer linearly projects local image patches to features at corresponding locations on the output map using shared parameters. While a convolution with random filters can project local patches to arbitrary output features, the output of a random linear projection approximately preserves relative similarity between input patches, proved in Appendix \ref{theorem_proof}. In other words, since any two locations within the same shape have similar local textures in the input image, they tend to be similar in the output feature map. Therefore, shapes that emerge in the output feature map are similar to shapes in the input image provided that the filter size is sufficiently small compared to the size of a typical shape.

In other words, the size of a convolution filter determines the smallest shape it can preserve. For example, 1x1 random convolutions preserve shapes at the single-pixel level and thus work as a random color mapping; large filters perturb shapes smaller than the filter size that are considered local texture of a shape at this larger scale. See Fig.~\ref{fig:randconv_example} for examples. \textit{More discussion and a formal proof are in Appendix \ref{sec:shapes} and \ref{theorem_proof}}. 

\subsection{Multi-scale Image Augmentation with a Randomized Convolution Layer}
\label{sec:randconv}

\begin{algorithm}[h]
	\small
	\caption{Learning with Data Augmentation by Random Convolutions}
	\label{Algorithm:randconv}
	\begin{algorithmic}[1]
		\State \textbf{Input}: Model $\Phi$, task loss $\mathcal{L}_{task}$, training images $\{I_i\}_{i=1}^N$ and their labels $\{y_i\}_{i=1}^N$, pool of filter sizes $\mathcal{K}=\{1,...,n\}$, fraction of original data $p$, whether to $\mathtt{mix}$ with original images, consistency loss weight $\lambda$
		\Function{$\text{{\RandConv}}$}{I, $\mathcal{K}$, $\mathtt{mix}$, $p$}
		\State Sample $p_0 \sim U(0, 1)$
		\If {$p_0$ < $p$ and $\mathtt{mix}$ is False} 
		\State return $I$ \Comment{When not in $\mathtt{mix}$ mode, use the original image with probability $p$}		
		\Else 
		\State Sample scale $k$ $\sim \mathcal{K}$
		\State Sample convolution weights $\mathbf{\Theta}\in\mathbb{R}^{k\times k \times 3\times 3} \sim N(0, \frac{1}{3 k^2})$
		\State $I_{rc} = I*\mathbf{\Theta}$ \Comment{Apply convolution on $I$}
		\If{$\mathtt{mix}$ is True}
		\State Sample $\alpha \sim U(0, 1)$
		\State return $\alpha I + (1-\alpha)I_{rc}$ \Comment{Mix with original images}
		\Else
		\State return $I_{rc}$
		\EndIf
		\EndIf
		\EndFunction
		\State \textbf{Learning Objective}:
		\For {$i = 1 \to N$}
		\For {$j = 1 \to 3$} 
		\State $\hat{y}_{i}^j = \Phi(\text{{\RandConv}}(I_i))$ \Comment{Predict labels for three augmented variants of the same image}
		
		\EndFor
		\State $\mathcal{L}_{cons} = \lambda\sum_{j=1}^{3}\text{KL}(\hat{y}^j_{i}|| \bar{y}_{i})$ where $\bar{y}_i = \sum_{j=1}^{3}\hat{y}_{i}^j/3$ \Comment{Consistency Loss}
		\State  $\mathcal{L} = \mathcal{L}_{task}(\hat{y}_{i}^1, y_i) + \lambda\mathcal{L}_{cons}$ \Comment{Learning with the task loss and the consistency loss}	
		\EndFor	
	\end{algorithmic}
\end{algorithm}

Sec.~\ref{sec:rc_shape_preserve} discussed how outputs of randomized convolution layers approximately maintain shape information at a scale larger than their filter sizes. Here, we develop our {\RandConv} data augmentation technique using a randomized convolution layer with $C_{out}=C_{in}$ to generate shape-consistent images with randomized texture (see Alg.~\ref{Algorithm:randconv}). {Our goal is not to use {\RandConv} to parameterize or represent texture as in previous filter-bank based texture models~\citep{heeger1995pyramid, portilla2000parametric}. Instead, we only use the three-channel outputs of {\RandConv} as new images with the same shape and different “style” (loosely referred to as "texture"). We also note that, a convolution layer is different from a convolution operation in image filtering. Standard image filtering applies the same 2D filter on three color channels separately. In contrast, our convolution layer applies three different \emph{3D} filters and each takes all color channels as input and generates one channel of the output.}
Our proposed {\RandConv} variants are as follows:

\textbf{{$\text{RC}_{\text{img}}$}: Augmenting Images with Random Texture} A simple approach is to use the randomized convolution layer outputs, $I*\mathbf{\Theta}$, as new images; where $\mathbf{\Theta}$ are the randomly sampled weights and $I$ is a training image. 
If the original training data is in the domain $D^0$, a sampled weight $\mathbf{\Theta}_k$ generates images with consistent global shape but random texture forming the random domain $D^k$. Thus, by random weight sampling, we obtain an infinite number of random domains $D^1,D^1,\dots, D^\infty$. Input image intensities are assumed to be a standard normal distribution $N(0, 1)$ (which is often true in practice thanks to data whitening). As the outputs of {\RandConv} should follow the same distribution, we sample the convolution weights from $N(0, \sigma^2)$ where $\sigma =1/\sqrt{C_{in}\times h \times w}$, which is commonly applied for network initialization~\citep{he2015delving}. We include the original images for training at a ratio $p$ as a hyperparameter.

\textbf{{$\text{RC}_{\text{mix}}$}: Mixing Variant}
As shown in Fig.~\ref{fig:randconv_example}, outputs from $\text{RC}_{\text{img}}$ can vary significantly from the appearance of the original images. Although generalizing to domains with significantly different local texture distributions is useful, we may not want to sacrifice much performance on domains similar to the training domain. Inspired by the AugMix ~\citep{hendrycks2020augmix} strategy, we propose to blend the original image with the outputs of the {\RandConv} layer via linear convex combinations $\alpha I + (1-\alpha)(I*\mathbf{\Theta})$, where $\alpha$ is the mixing weight uniformly sampled from $[0,1]$.
In $\text{RC}_{\text{mix}}$, the {\RandConv} outputs provide shape-consistent perturbations of the original images. Varying $\alpha$, we continuously interpolate between the training domain and the randomly sampled domains of \texttt{$\text{RC}_{\text{img}}$}.

\textbf{Multi-scale Texture Corruption} 
As discussed in Sec.~\ref{sec:rc_shape_preserve},, image shape information at a scale smaller than a filter's size will be corrupted by {\RandConv}. Therefore, we can use filters of varying sizes to preserve shapes at various scales.
We choose to uniformly randomly sample a filter size $k$ from a pool $\mathcal{K}={1,3,...n}$ before sampling convolution weights $\mathbf{\Theta}\in\mathbb{R}^{k\times k \times C_{in}\times C_{out}}$ from a Gaussian distribution $N(0, \frac{1}{k^2C_{in}})$. Fig.~\ref{fig:randconv_example} shows examples of multi-scale {\RandConv} outputs.

\textbf{Consistency Regularization} To learn representations invariant to texture changes, we use a loss encouraging consistent network predictions for the same {\RandConv}-augmented image for different random filter samples. Approaches for transform-invariant domain randomization~\citep{yue2019domain}, data augmentation~\citep{hendrycks2020augmix}, and semi-supervised learning~\citep{berthelot2019mixmatch} use similar strategies. We use Kullback-Leibler (KL) divergence to measure consistency. However, enforcing prediction similarity of two augmented variants may be too strong. Instead, following ~\citep{hendrycks2020augmix}, we use {\RandConv} to obtain 3 augmentation samples of image $I$: $G_j = \text{{\RandConv}}^j(I)$ for $j=1,2,3$ and obtain their predictions with a model $\Phi$: $y^j = \Phi(G^j)$. We then compute the \emph{relaxed} loss as $\lambda\sum_{j=1}^{3}\text{KL}(y^j|| \bar{y})$, where $\bar{y} = \sum_{j=1}^{3}y^j/3$ is the sample average.

\section{Experiments}
Secs.~\ref{section:digits} to~\ref{section:imagenet-sketch} evaluate our methods on the following datasets: multiple digit recognition datasets, PACS, and ImageNet-sketch. Sec.~\ref{section:revisit_pacs} uses PACS to explore the out-of-domain generalization of a pretrained representation in transfer learning by checking if pretraining on ImageNet with our method improves the domain generalization performance in downstream tasks. All experiments are in the single-domain generalization setting where training and validation sets are drawn from one domain. \textit{Additional experiments with ResNet18 as the backbone are given in the Appendix}.

\subsection{Digit Recognition}
\label{section:digits}

The five digit recognition datasets (MNIST~\citep{lecun1998gradient}, MNIST-M~\citep{ganin2016domain}, SVHN~\citep{netzer2011reading}, SYNTH~\citep{ganin2014unsupervised} and USPS \citep{denker1989neural}) have been widely used for domain adaptation and generalization research~\citep{peng2019moment, Peng2019DomainAL,qiao2020learning}. Following the setups in~\citep{volpi2018generalizing} and~\citep{qiao2020learning}, we train a simple CNN with \emph{10,000} MNIST samples and evaluate the accuracy on the test sets of the other four datasets. We also test on MNIST-C~\citep{mu2019mnist}, a robustness benchmark with \emph{15 common corruptions} of MNIST and report the average accuracy over all corruptions.

\begin{figure}[b]
	\begin{center}
		\setlength{\tabcolsep}{0.00cm}
		\newcommand\cwidth{0.4\textwidth}
		\begin{adjustbox}{max width=\textwidth}
			\begin{tabular}{ccc}
				\includegraphics[width=\cwidth]{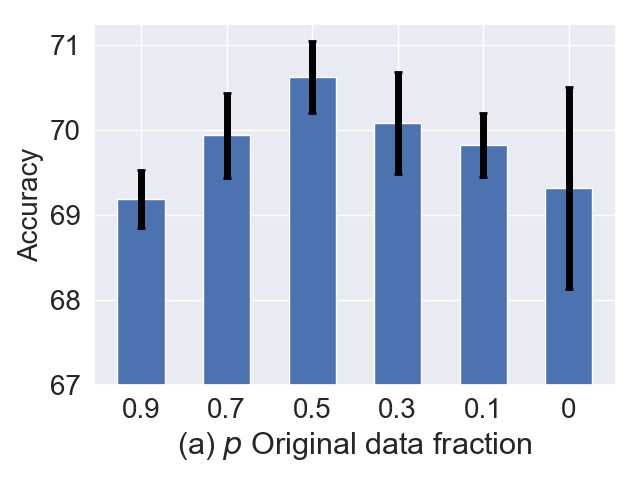} &
				\includegraphics[width=\cwidth]{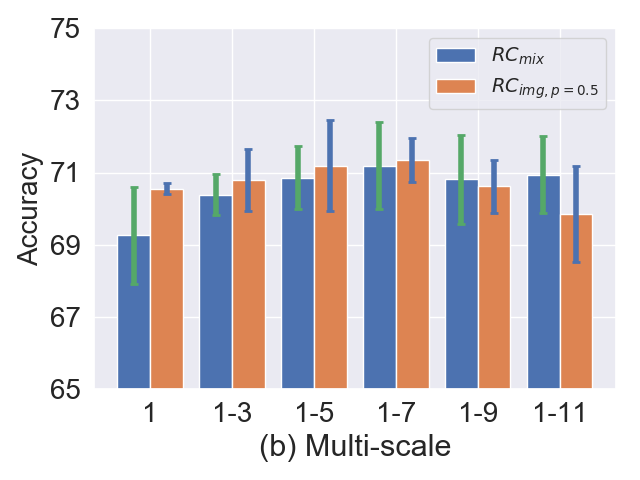} &
				\includegraphics[width=\cwidth]{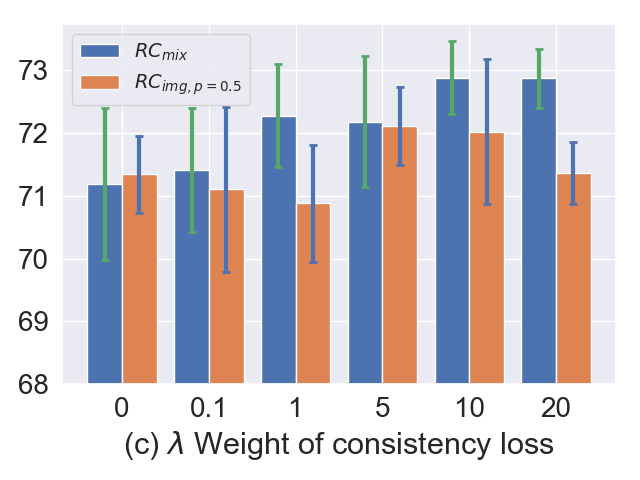}\\
			\end{tabular}
		\end{adjustbox}
	\end{center}
	\vspace{-4mm}
	\caption{Average accuracy and 5-run variance of MNIST model on MNIST-M, SVHN, SYNTH and USPS. Studies for: (a) original data fraction $p$ for $\text{RC}_{\text{img}}$; (b) multiscale design (1-n refers to using scales {1,3,..,n}) for $\text{RC}_{\text{img},p=0.5}$ (orange) and $\text{RC}_{\text{mix}}$ (blue); (c) consistency loss weight $\lambda$ for  $\text{RC}_{\text{img}1-7,p=0.5}$ (orange) and $\text{RC}_{\text{mix}1-7}$ (blue).}  
	\label{fig:digit_ablation}
\end{figure}

\textbf{Selecting Hyperparameters and Ablation Study.} Fig. \ref{fig:digit_ablation}(a) shows the effect of the hyperparameter $p$ on $\text{RC}_{\text{img}}$ with filter size 1. We see that  adding only $10\%$ {\RandConv} data ($p=0.9$) immediately improves the average performance (DG-Avg) on MNIST-M, SVHN, SYNTH and USPS performance from 53.53 to 69.19, outperforming all other approaches (see Tab.~\ref{table:digits_compare}) for every dataset. 
We choose $p=0.5$, which obtains the best DG-Avg. Fig.~\ref{fig:digit_ablation}(b) shows results for a multiscale ablation study. Increasing the pool of filter sizes up to $7$ improves DG-Avg performance. Therefore we use multi-scale $1\mhyphen7$ to study the consistency loss weight $\lambda$, shown in Fig. ~\ref{fig:digit_ablation}(c). Adding the consistency loss improves both {\RandConv} variants on 
DG-avg: $\text{RC}_{\text{mix}1-7}$ favors $\lambda=10$ while $\text{RC}_{\text{img}1-7,p=0.5}$ performs similarly for $\lambda=5$ and $\lambda=10$. We choose $\lambda=10$ for all subsequent experiments.

\textbf{Results.} Tab.~\ref{table:digits_compare} compares the performance of $\text{RC}_{\text{img}1-7,p=0.5,\lambda=10}$ and $\text{RC}_{\text{mix}1-7,\lambda=10}$ with other state-of-the-art approaches. We show results of the adversarial training based methods GUD~\citep{volpi2018generalizing}, M-ADA~\citep{qiao2020learning}, and PAR~\citep{wang2019learning}. The baseline model is trained only on the standard classification loss. 
To show {\RandConv} is more than a trivial color/contrast adjustment method, we also compare to ColorJitter\footnote{See PyTorch documentation for implementation details; all parameters are set to 0.5.} data augmentation (which randomly changes image brightness, contrast, and saturation) and GreyScale (where images are transformed to grey-scale for training and testing). {We also tested data augmentation with a fixed Laplacian of Gaussian filter (Band-Pass) of size=3 and $\sigma=1$ and the data augmentation pipeline (Multi-Aug) that was used in a recently proposed large scale study on domain generalization algorithms and datasets \citep{gulrajani2020search}. } {\RandConv} and its mixing variant outperforms the best competing method (M-ADA) by 17\% on DG-Avg and achieves the best 91.62\% accuracy on MNIST-C. While the difference between the two variants of {\RandConv} is marginal, $\text{RC}_{\text{mix}1-7,\lambda=10}$ performs better on both DG-Avg and MNIST-C. {When combined with Multi-Aug, {\RandConv} achieves improved performance except on MNIST-C.}
Fig~\ref{fig:feat_tsne} shows t-SNE image feature plots for unseen domains generated by the baseline approach and $\text{RC}_{\text{mix}1-7,\lambda=10}$. The {\RandConv} embeddings suggest better generalization to unseen domains.

\begin{table}[th]
	\small
	\setlength{\tabcolsep}{4pt}
	\caption{Average accuracy and 5-run standard deviation (in parenthesis) of MNIST10K model on MNIST-M, SVHN, SYNTH, USPS and their average (DG-avg); and average accuracy of 15 types of corruptions in MNIST-C. Both {\RandConv} variants significantly outperform all other methods.}
	\label{table:digits_compare}
	\centering
	\begin{tabular}{l|l|lllll|l}
		\toprule
		& MNIST   & MNIST-M    & SVHN        & USPS        & SYNTH       & DG-Avg         & MNIST-C     \\
		\midrule
		Baseline    & 98.40\tiny(0.84) & 58.87\tiny(3.73) & 33.41\tiny(5.28) & 79.27\tiny(2.70) & 42.43\tiny(5.46) & 53.50\tiny(4.23) & 88.20\tiny(2.10) \\
		
		GreyScale   & 98.82\tiny(0.02) & 58.41\tiny(0.99)          & 36.06\tiny(1.48)          & 80.45\tiny(1.00)          & 45.00\tiny(0.80)          & 54.98\tiny(0.86)          & 89.15\tiny(0.44)          \\
		
		ColorJitter   & 98.72\tiny(0.05) & 62.72\tiny(0.66)          & 39.61\tiny(0.88)          & 79.18\tiny(0.60)          & 46.40\tiny(0.34)          & 56.98\tiny(0.39)          & 89.48\tiny(0.18)          \\
		
		{BandPass} & 98.65\tiny(0.11) &70.22\tiny(2.73) &	48.34\tiny(2.56)	& 78.60\tiny(0.82) &	57.17\tiny(2.01) &	63.58\tiny(1.89)	& 87.89\tiny(0.68) \\
		
		{MultiAug} & 98.80\tiny(0.05)	&62.32\tiny(0.66)	& 39.07\tiny(0.68)	&79.31\tiny(1.02)	 &46.48\tiny(0.80)&	56.79\tiny(0.34)	& 89.54\tiny(0.11) \\
		
		PAR (our imp)       & 98.79\tiny(0.05) & 61.16\tiny(0.21) & 36.08\tiny(1.27) & 79.95\tiny(1.18) & 45.48\tiny(0.35) & 55.67\tiny(0.33) & 89.34\tiny(0.45) \\

		GUD         & -           & 60.41       & 35.51       & 77.26       & 45.32       & 54.62       & -           \\
		M-ADA       & -           & 67.94       & 42.55       & 78.53       & 48.95       & 59.49       & -           \\
		\midrule

		$\text{RC}_{\text{img}{1\mhyphen7}}$\tiny, $p$=0.5, $\lambda$=5    & 98.86\tiny(0.05) & 87.67\tiny(0.37)          & 54.95\tiny(1.90)          & 82.08\tiny(1.46)          & {63.37\tiny(1.58)}          & 72.02\tiny(1.15)          & 90.94\tiny(0.51)\\
		$\text{RC}_{\text{mix}1\mhyphen7,\lambda=10}$ & 98.85\tiny(0.04) & 87.76\tiny(0.83) & {57.52\tiny(2.09)} & {83.36\tiny(0.96)} & 62.88\tiny(0.78) & {72.88\tiny(0.58)} & \textbf{91.62\tiny(0.77)} \\
		{$\text{RC}_{\text{mix}1\mhyphen7,\lambda=10}$ + \scriptsize MultiAug} & 98.82\tiny(0.06) & \textbf{87.89\tiny(0.29)} & \textbf{62.07\tiny(0.62)} & \textbf{84.39\tiny(1.02)} & \textbf{63.90\tiny(0.63)} & \textbf{74.56\tiny(0.46)} &  91.40\tiny(0.93) \\
		
		\bottomrule
	\end{tabular}
\end{table}

\begin{figure}[th]
	\begin{center}
		\small
		\newcommand\cwidth{0.25\textwidth}
		\begin{adjustbox}{max width=\textwidth}
			\begin{tabular}{cccc}
				MNIST-M    & SVHN        & USPS   & SYNTH \\
				\includegraphics[width=\cwidth]{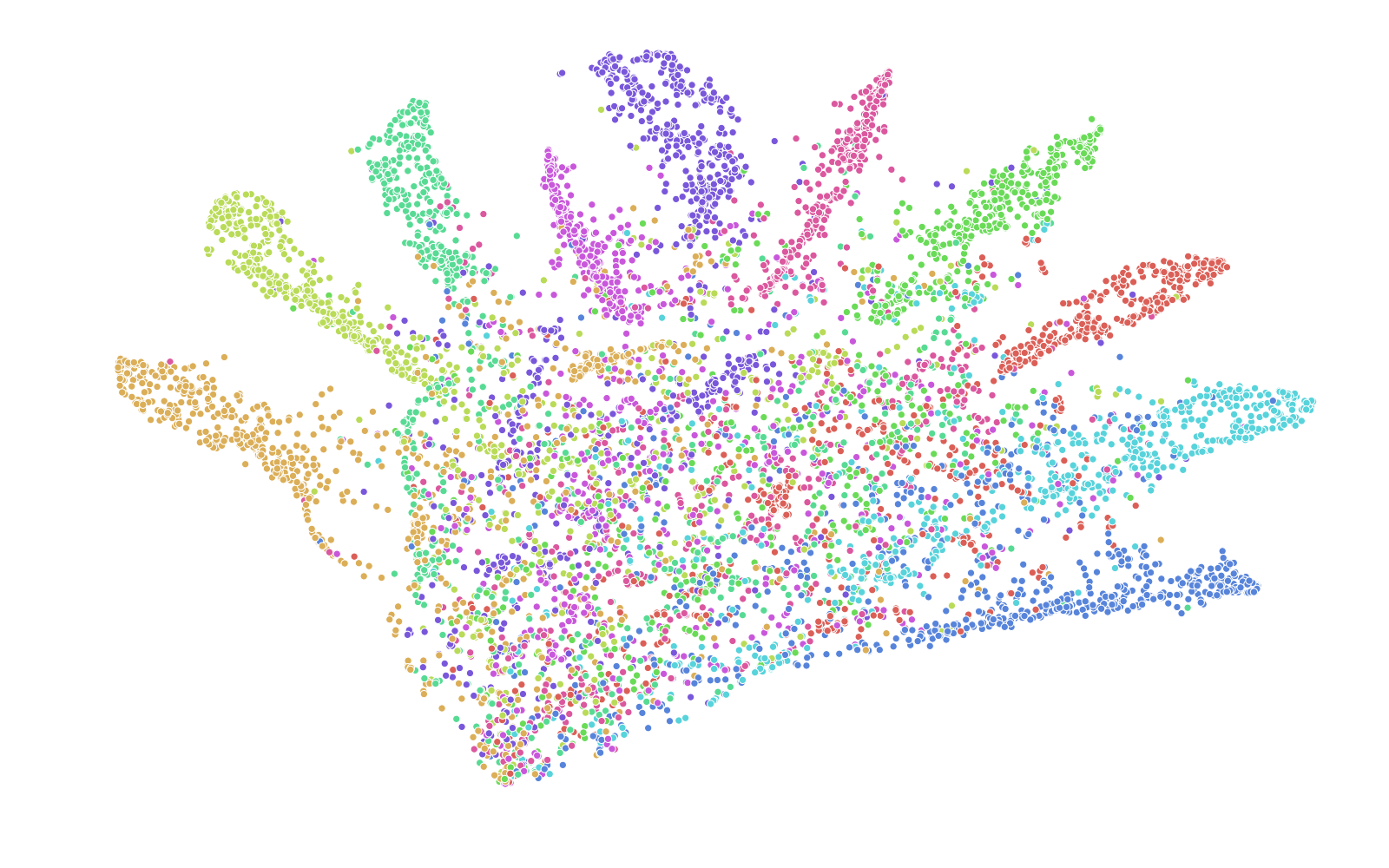}
				& \includegraphics[width=\cwidth]{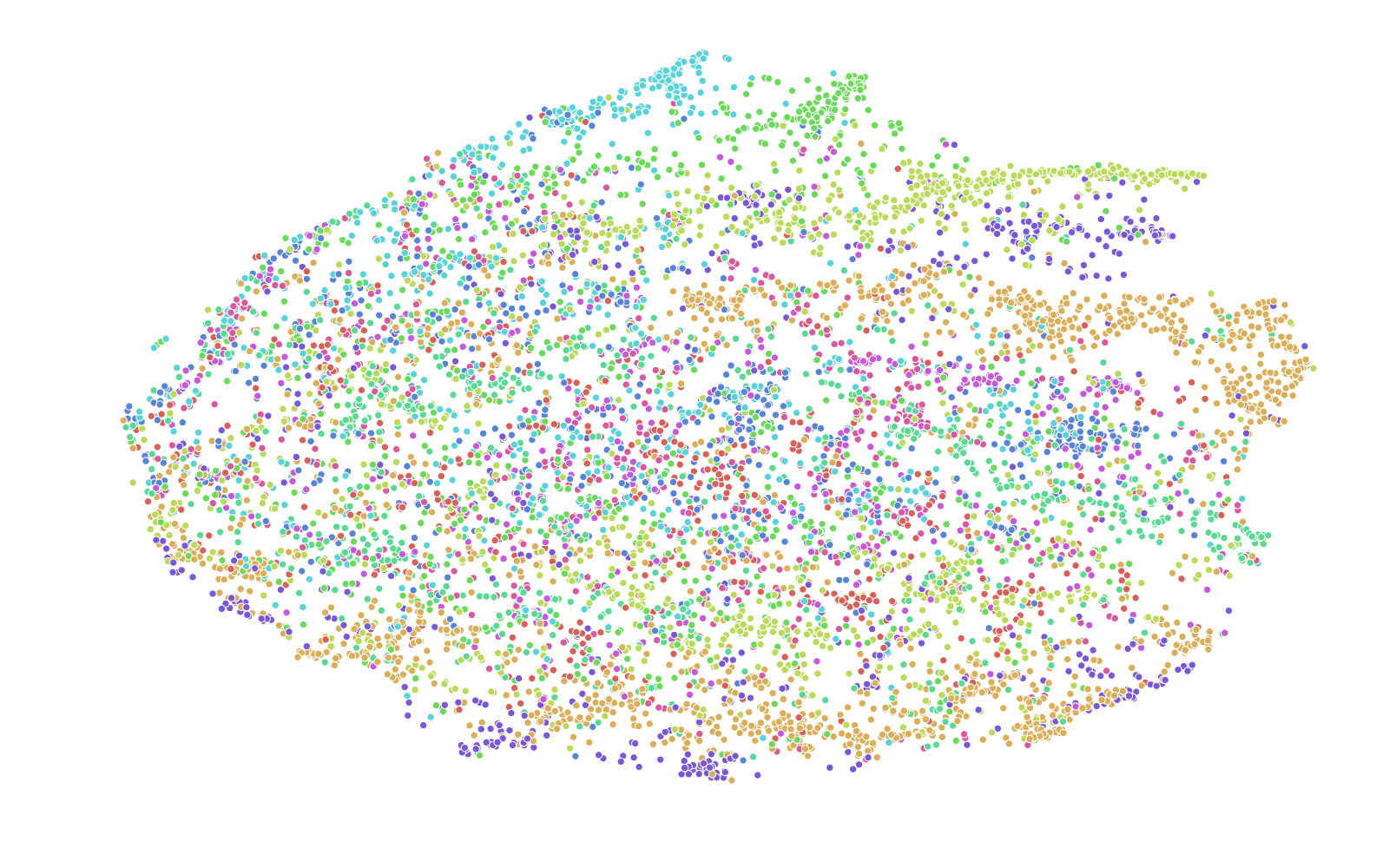}
				& \includegraphics[width=\cwidth]{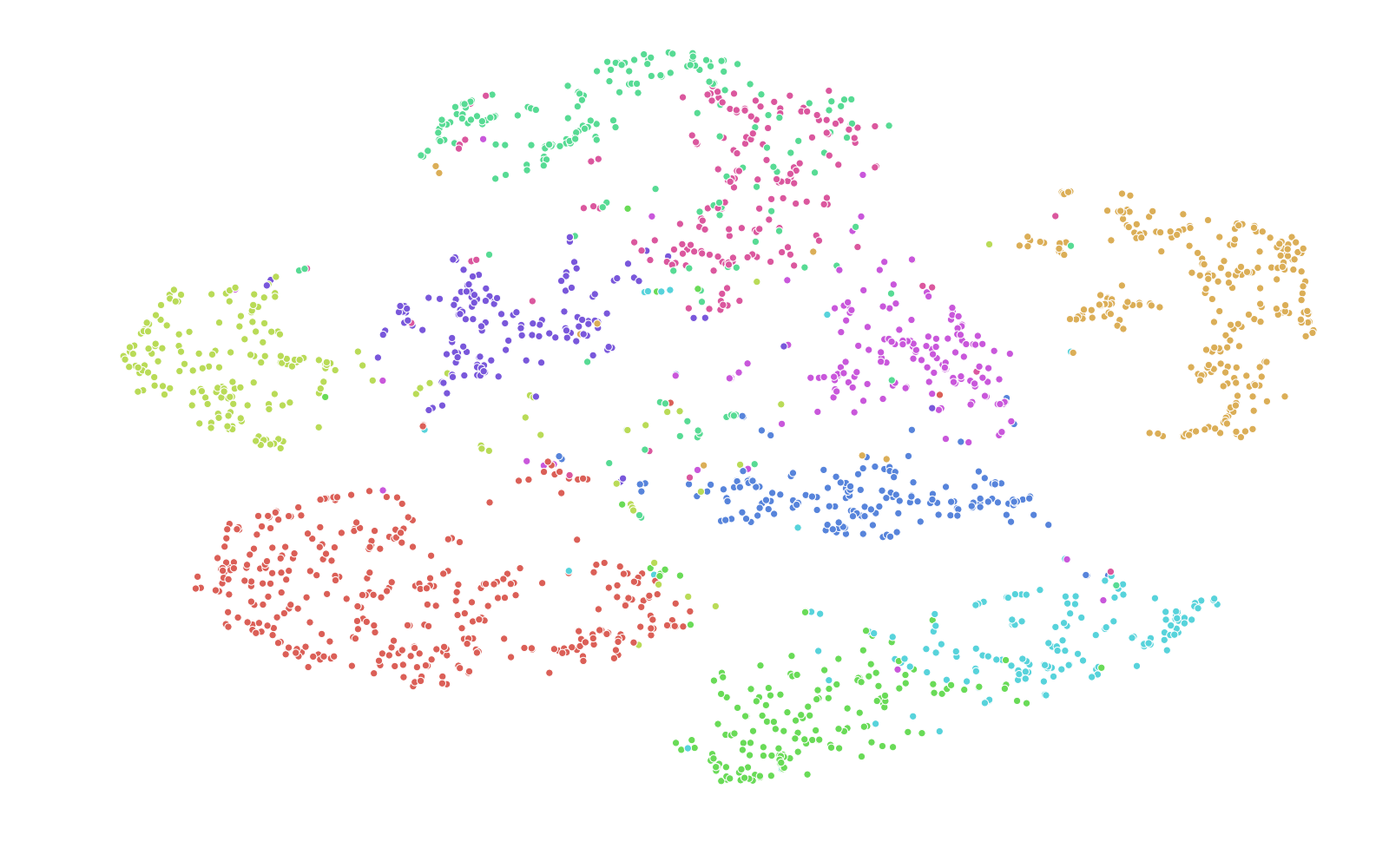}
				& \includegraphics[width=\cwidth]{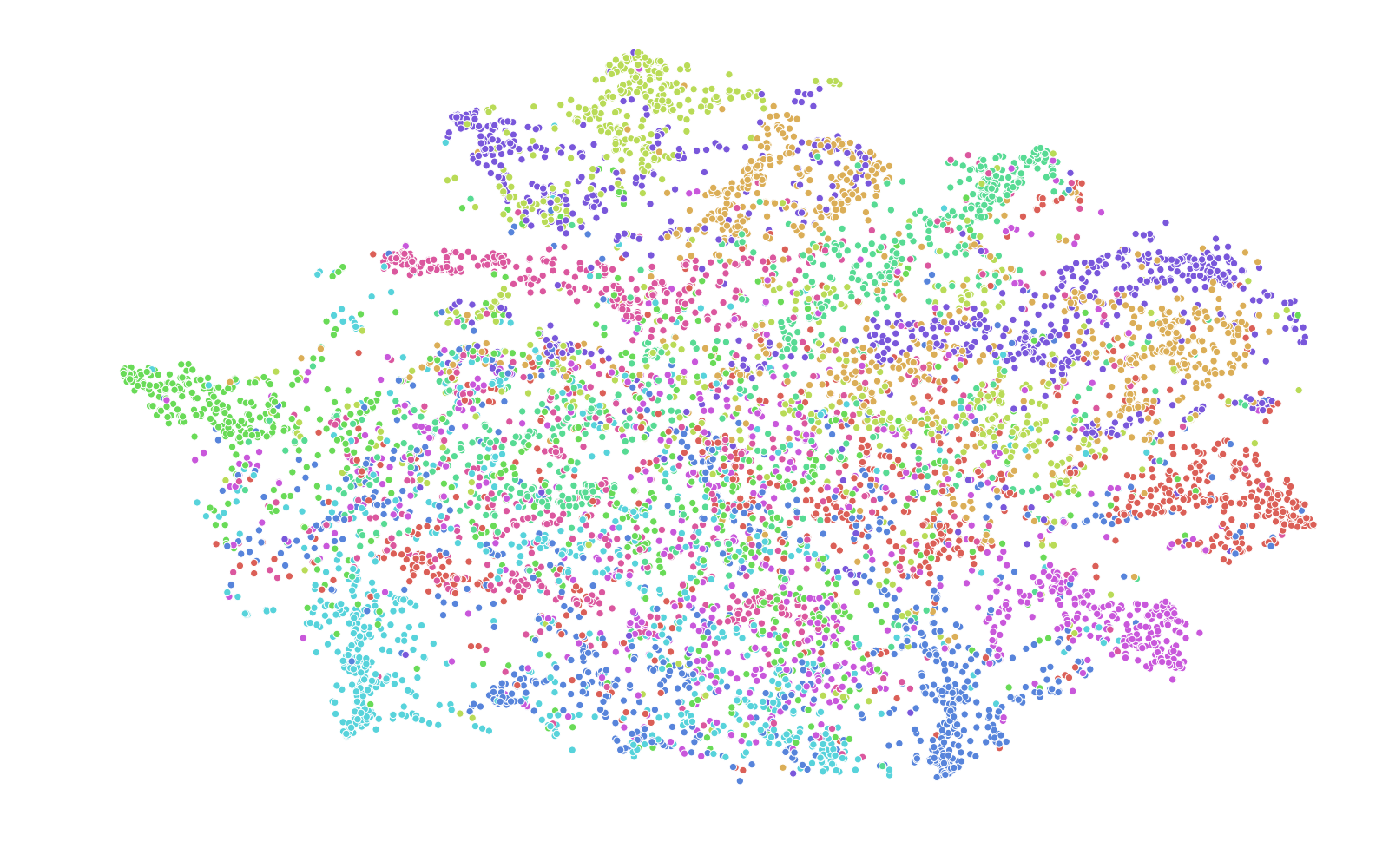}
				\\
				\includegraphics[width=\cwidth]{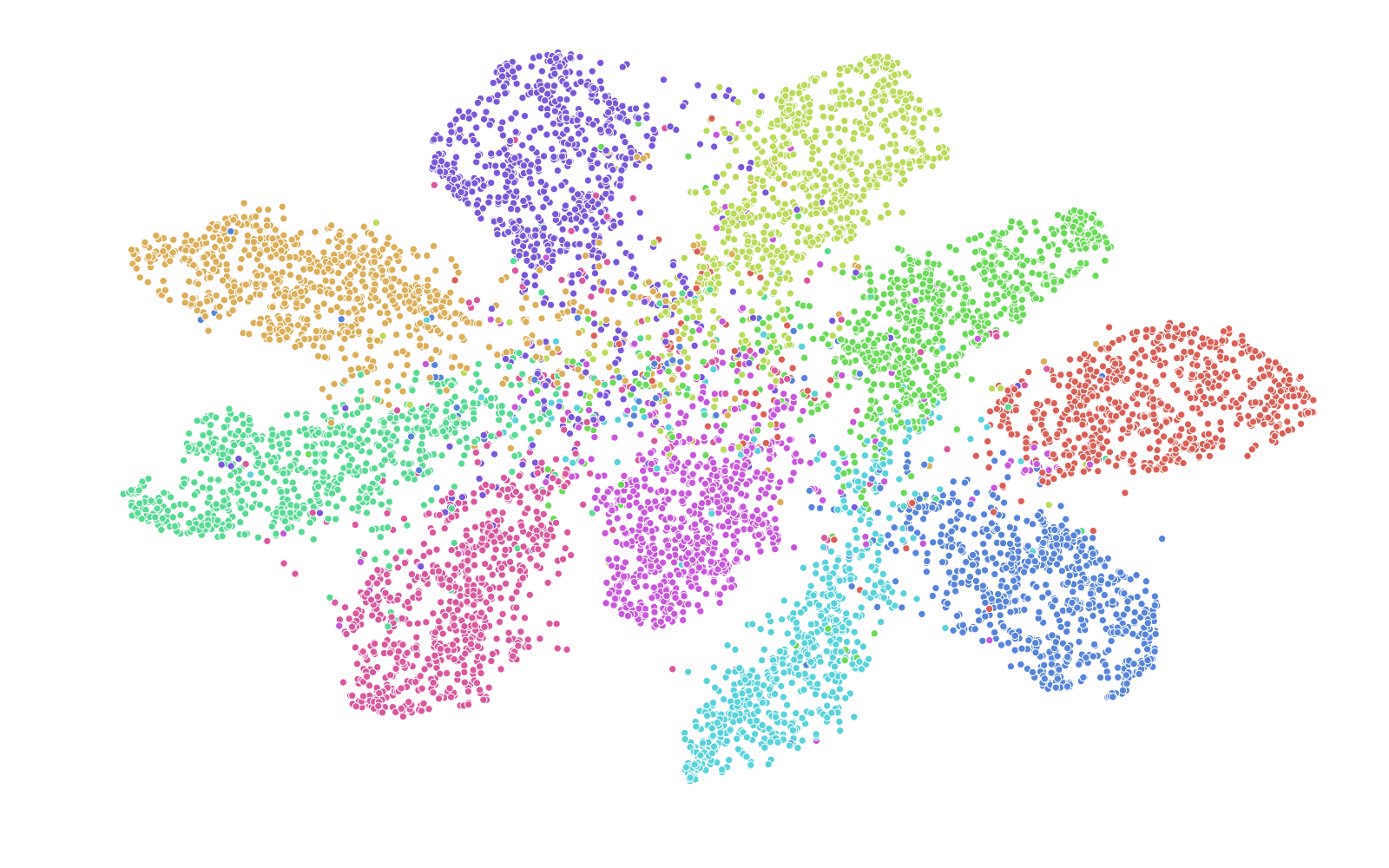}
				& \includegraphics[width=\cwidth]{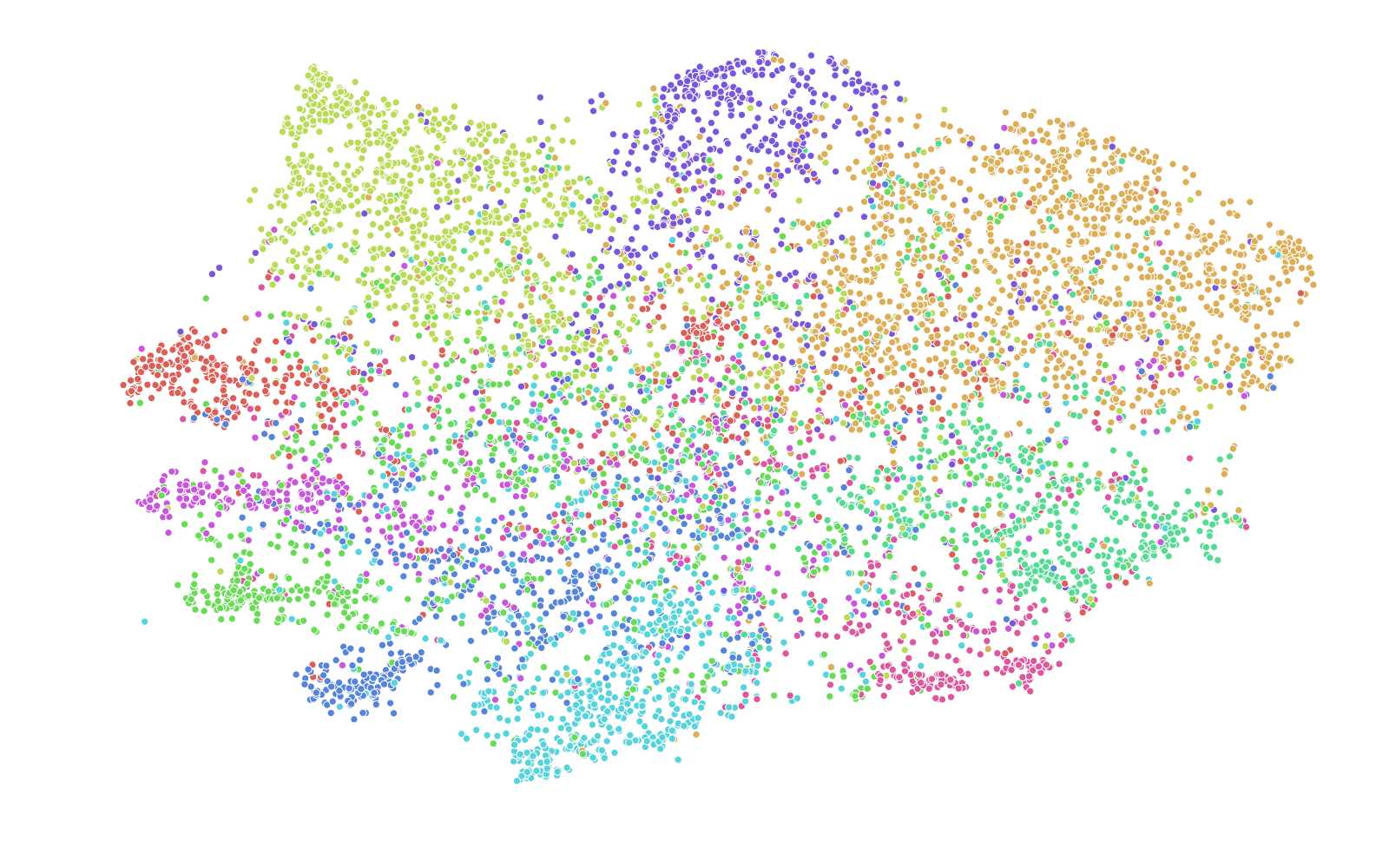}
				& \includegraphics[width=\cwidth]{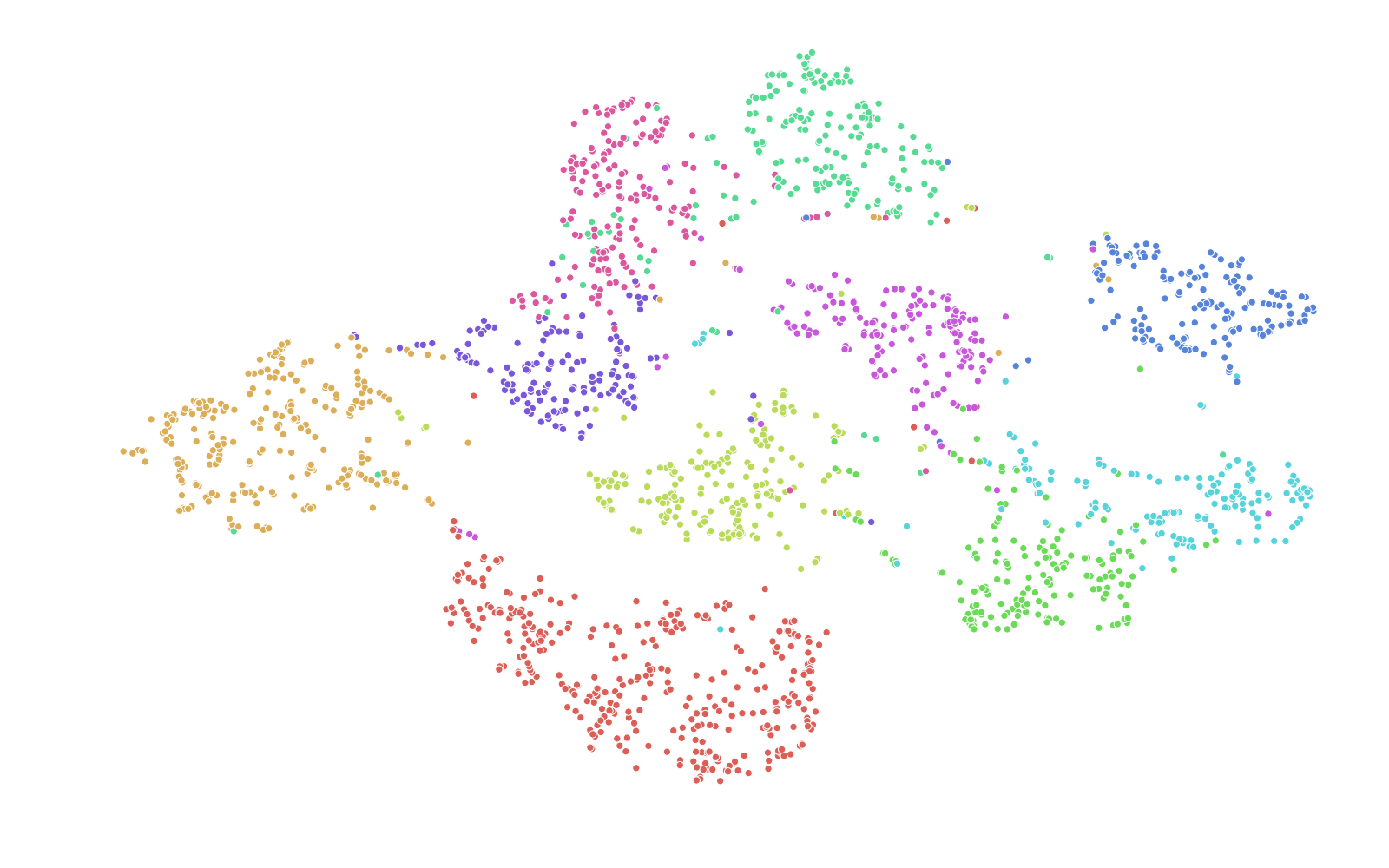}
				& \includegraphics[width=\cwidth]{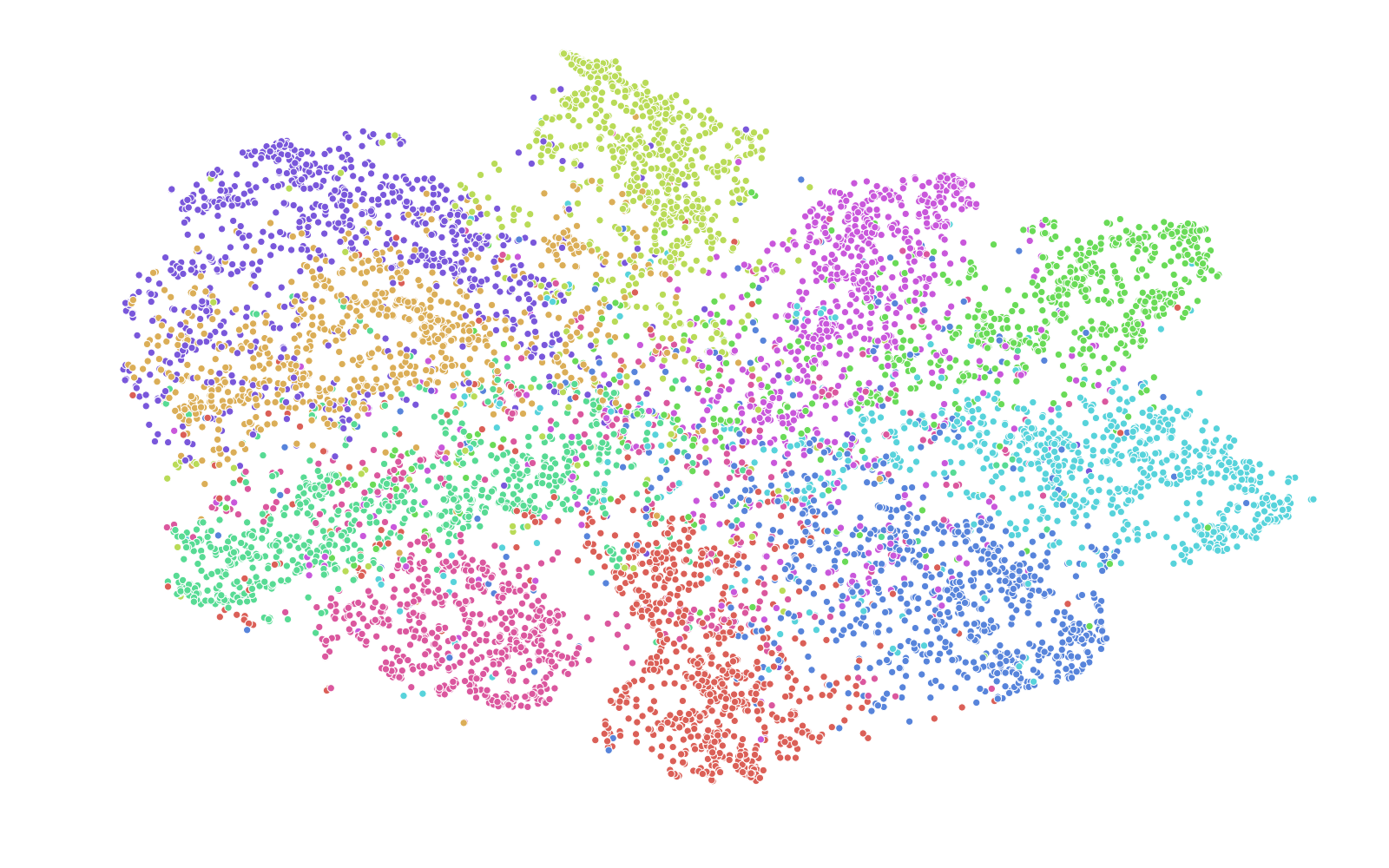}
				\\
			\end{tabular}
		\end{adjustbox}
	\end{center}
	\caption{t-SNE feature embedding visualization for digit datasets for models trained on MNIST without (top) and with our $\text{RC}_{\text{mix}{1\mhyphen7},\lambda=10}$ approach (bottom). Different colors denote different classes.}  
	\label{fig:feat_tsne}
\end{figure}

\begin{table}[t]
	\small
	\caption{Mean and 5-run standard deviation (in parenthesis) results for domain generalization on PACS. Best results with our Deep-All baseline are in \textbf{bold}. The  domain name in each column represents the target domain. {Base column indicates different baselines and results under different baselines are not directly comparable. MLDG and CIDDF used domain labels for training.}}
	\label{table:PACS}
	\centering
	\begin{tabular}{c|lccccc}
		\toprule
		Base & Method & Photo & Art & Cartoon & Sketch & Average \\
		\toprule
		\multirow{10}{*}{\centering Ours}& Deep-All     & 86.77\tiny(0.42) & 60.11\tiny(1.33) & 64.12\tiny(0.32) & 55.28\tiny(4.71) & 66.57\tiny(1.36) \\
		
		& GreyScale        & 83.93\tiny(1.47) & 61.60\tiny(1.18) & 62.12\tiny(0.61) & 60.07\tiny(2.47) & 66.93\tiny(0.83) \\
		
		& ColorJitter & 84.61\tiny(0.83) & 59.01\tiny(0.24) & 61.43\tiny(0.68) & 62.44\tiny(1.68) & 66.88\tiny(0.33) \\
		
		&{BandPass} & 87.08\tiny(0.57) & 	59.46\tiny(0.27)	& \textbf{64.39\tiny(0.51)}	& 55.39\tiny(2.95)&	66.58\tiny(0.73) \\
		
		&{MultiAug} & 85.21\tiny(0.47)	& 59.51\tiny(0.38)&	62.88\tiny(1.01)&	61.67\tiny(0.76)&	67.32\tiny(0.23) \\
		
		&PAR (our imp.) & \textbf{87.21\tiny(0.42)} & 60.17\tiny(0.95)  & 63.63\tiny(0.88)  & 55.83\tiny(2.57) & 66.71\tiny(0.58)\\
		\cmidrule{2-7}
		
		&$\text{RC}_{\text{img}{1\mhyphen7}}$\tiny, $p$=0.5 & 86.50\tiny(0.72) & 61.10\tiny(0.38) &  {64.24\tiny(0.62)} & 68.50\tiny(1.83) & 70.09\tiny(0.43) \\
		
		&$\text{RC}_{\text{mix}{1\mhyphen7}}$ & 86.60\tiny(0.67) &  {61.74\tiny(0.90)} & 64.05\tiny(0.66) & 69.74\tiny(0.66) &  \textbf{70.53\tiny(0.25)} \\
		
		& {$\text{RC}_{\text{mix}{1\mhyphen7}}$ + \scriptsize MultiAug } & 86.23\tiny(0.74)&	\textbf{61.91\tiny(0.76)}&	62.69\tiny(0.76)&	67.74\tiny(1.21)&	69.64\tiny(0.49) \\
		
		&$\text{RC}_{\text{img}{1\mhyphen7}}$\tiny, $p$=0.5, $\lambda$=10  & 81.15\tiny(0.76) & 59.56\tiny(0.79) & 62.42\tiny(0.59) & 71.74\tiny(0.43) & 68.72\tiny(0.58) \\	
		
		&$\text{RC}_{\text{mix}{1\mhyphen7}}$\tiny,$\lambda$=10 & 81.78\tiny(1.11) & 61.14\tiny(0.51) & 63.57\tiny(0.29) & \textbf{71.97\tiny(0.38)} & 69.62\tiny(0.24) \\

		\midrule
		\multicolumn{7}{c}{{Results below are not directly comparable due to different Deep-All implementations.}} \\
		\midrule
		\multirow{3}{*}{\scriptsize	\citet{wang2019learning}} & Deep-All (our run) & 88.40 & 66.26 & 66.58 & 59.40 & 70.16 \\
		& PAR (our run) & 88.40 & 65.19 & 68.58  & 61.86 & 71.10 \\
		& PAR (reported) & 89.6 & 66.3 & 68.3 & 64.1 & 72.08 \\
		\midrule
		\multirow{2}{*}{\scriptsize\citet{carlucci2019jigen}} & Deep-All & 89.98 & 66.68 & 69.41 & 60.02  & 71.52 \\
		& Jigen & 89.00 & 67.63& 71.71& 65.18 & 73.38 \\
		
		\midrule
		\multirow{2}{*}{\scriptsize\citet{li2018mldg}} & Deep-All &  86.67 & 64.91 & 64.28 & 53.08  & 67.24 \\
		& MLDG (\tiny use domain labels) & 88.00& 66.23 & 66.88 & 58.96  & 70.01 \\
		\midrule
		\multirow{2}{*}{\scriptsize\citet{li2018ciddg}} & Deep-All& 77.98  &  57.55  & 67.04  & 58.52  & 65.27 \\
		& CIDDG (\tiny use domain labels)& 78.65 & 62.70 & 69.73 & 64.45  & 68.88 \\
		\bottomrule
	\end{tabular}
\end{table}

\subsection{PACS Experiments}

\label{section:PACS}
The PACS dataset~\citep{li2018domain} considers 7-class classification on 4 domains: photo, art painting, cartoon, and sketch, with very different texture styles. Most recent domain generalization work studies the multi-source domain setting on PACS and uses domain labels of the training data. Although we follow the convention to train on 3 domains and to test on the fourth, we simply pool the data from the 3 training domains as in~\citep{wang2019learning}, without using domain labels during the training. 

\textbf{Baseline and State-of-the-Art}. Following~\citep{li2017deeper}, we use Deep-All as the baseline, which finetunes an ImageNet-pretrained AlexNet on 3 domains using only the classification loss and tests on the fourth domain. We test our {\RandConv} variants  $\text{RC}_{\text{img}{1\mhyphen7},p=0.5}$ and $\text{RC}_{\text{mix}{1\mhyphen7}}$ with and without consistency loss, and ColorJitter/GreyScale/{BandPass/MultiAug} data augmentation as in the digit datasets. We also implemented PAR~\citep{wang2019learning} using our baseline model. { $\text{RC}_{\text{mix}{1\mhyphen7}}$ combined with MultiAug is also tested.}
Further, we compare to the following state-of-the-art approaches: Jigen~\citep{carlucci2019jigen} using self-supervision, MLDG~\citep{li2018mldg} using meta-learning, 
and the conditional invariant deep domain generalization method CIDDG~\citep{li2018ciddg}. Note that {previous methods used different Deep-All baselines which make the final accuracy not directly comparable, and} MLDG and CIDDG use domain labels for training.

\textbf{Results.} Tab.~\ref{table:PACS} shows \emph{significant improvements on Sketch} for both {\RandConv} variants. Sketch is the most challenging domain with no color and much less texture compared to the other 3 domains. The success on Sketch demonstrates that our methods can guide the DNN to learn global representations focusing on shapes that are robust to texture changes. Without using the consistency loss, $\text{RC}_{\text{mix}{1\mhyphen7}}$ achieves the best overall result improving over Deep-All by $\sim$4\% {but adding MultiAug does not further improve the performance}.
Adding the consistency loss with $\lambda=10$, $\text{RC}_{\text{mix}{1\mhyphen7}}$ and $\text{RC}_{\text{img}{1\mhyphen7}, p=0.5}$ performs better on Sketch but degrades performance on the other 3 domains, so do GreyScale and ColorJitter. \textit{This observation will be discussed in Sec~\ref{section:revisit_pacs}}.


\subsection{{Generalizing an ImageNet Model to ImageNet-Sketch}}
\label{section:imagenet-sketch}

\begin{table}[htp]
	\vspace{-4mm}
	\small
	\caption{Accuracy of ImageNet-trained AlexNet on ImageNet-Sketch (IN-S) data. Our methods outperform PAR by 5\% {and are on par with a Stylized-ImageNet (SIN) trained model. Note that PAR was built on top of a stronger baseline than our model, and both PAR and SIN fine-tuned the baseline model which helped the performance, while we train \texttt{RandConv} model from scratch.}}
	\label{table:imagenet}
	\centering
	\begin{tabular}{c|cc|ccc|c}
		\toprule
		      & \multirow{2}{*}{\shortstack{Baseline\\ \tiny\citep{wang2019learning}}} & \multirow{2}{*}{\shortstack{PAR\\ \tiny\citep{wang2019learning}}}   & \multirow{2}{*}{Baseline} &
		\multirow{2}{*}{\shortstack{$\text{RC}_{\text{img}1\mhyphen7,}$ \\\tiny $p$=0.5,$\lambda$=10}}
		& \multirow{2}{*}{\shortstack{$\text{RC}_{\text{mix}1\mhyphen7,}$ \\\tiny $\lambda$=10 }} & 
		
		\multirow{2}{*}{{\shortstack{SIN\\\tiny\citep{geirhos2018imagenettrained}}}} \\
		\\
		\toprule
            Top1 & 12.04  & 13.06 & 10.28  &
		18.09 & 16.91 & 17.62 \\
		Top5  & 25.60 & 26.27 & 21.60 
		& 35.40 & 33.99 & 36.22\\ 
		\bottomrule
	\end{tabular}
\end{table}

ImageNet-Sketch~\citep{wang2019learning} is an out-of-domain test set for models trained on ImageNet. We trained AlexNet from scratch with $\text{RC}_{\text{img}1\mhyphen7, p=0.5, \lambda=10}$ and $\text{RC}_{\text{mix}1\mhyphen7,\lambda=10}$. We evaluate their performance on ImageNet-Sketch. We use the AlexNet model trained without {\RandConv} as our baseline. Tab.~\ref{table:imagenet} compares \texttt{PAR} and its baseline model and {AlexNet trained with Stylized ImageNet (SIN) \citep{geirhos2018imagenettrained} on ImageNet-Sketch.}
Although \texttt{PAR} uses a stronger baseline, {\RandConv} achieves significant improvements over our baseline and outperforms \texttt{PAR} by a large margin. Our methods achieve more than a 7\% accuracy improvement over the baseline and surpass PAR by 5\%. 
 {SIN as an image stylization approach that can modify image texture in a hierarchical and realistic way. However, albeit its complexity, it still performs on par with RandConv. Note that image stylization techniques require additional data and heavy precomputation. Further, the images for the style source also need to be chosen. In contrast, RandConv is much easier to use: it can be applied to any dataset via a simple convolution layer. We also measure the shape-bias metric proposed by
\cite{geirhos2018imagenettrained} for \texttt{RandConv} trained AlexNet. $\text{RC}_{\text{img}1\mhyphen7, p=0.5, \lambda=10}$ and $\text{RC}_{\text{mix}1\mhyphen7,\lambda=10}$ improve the baseline from $25.36\%$ to $48.24\%$ and $54.85\%$ respectively.}


\subsection{Revisiting PACS with more Robust Pretrained Representations}
\label{section:revisit_pacs}
A common practice for many computer vision tasks (including the PACS benchmark) is transfer learning, i.e.\ finetuning a backbone model pretrained on ImageNet. Recently, how the accuracy on ImageNet \citep{kornblith2019better} and adversial robustness \citep{salman2020adversarially} of the pretrained model affect transfer learning has been studied in the context of domain generalization. Instead, we study how out-of-domain generalizability transfers from pretraining to downstream tasks and shed light on how to better use pretrained models.

\textbf{Impact of ImageNet Pretraining} A model trained on ImageNet may be biased towards textures~\citep{geirhos2018imagenettrained}. 
Finetuning ImageNet pretrained models on PACS may inherit this texture bias, thereby benefitting generalization on the Photo domain (which is similar to ImageNet), but hurting performance on the Sketch domain. Therefore, as shown in Sec.~\ref{section:PACS}, using {\RandConv} to correct this texture bias improves results on Sketch, but degrades them on the Photo domain. 
Since pretraining has such a strong impact on transfer performance to new tasks, we ask: \emph{"Can the generalizability of a pretrained model transfer to downstream tasks? I.e., does a pretrained model with better generalizability improve performance on unseen domains on new tasks?"} To answer this, we revisit the PACS tasks based on  ImageNet-pretrained weights where our two {\RandConv} variants of Sec.~\ref{section:imagenet-sketch} are used during ImageNet training. We study if this results in performance changes for the Deep-All baseline and for finetuning with {\RandConv}.

\begin{table}[t]
	\small
	\caption{Generalization results on PACS with {\RandConv} and SIN pretrained AlexNet.  ImageNet column shows how the pretrained model is trained on ImageNet (baseline represents training the ImageNet model using only the classification loss); PACS column indicates the methods used for finetuning on PACS. \textbf{Best} and \underline{second best} accuracy for each target domain are highlighted in bold and underlined.}
	\label{table:transfer_PACS}
	\centering
	\begin{tabular}{c|c|ccccc}
		\toprule
		PACS                & ImageNet & Photo & Art & Cartoon & Sketch          & Avg                  \\
		\toprule
		\multirow{4}{*}{Deep-All} & 
		Baseline     & \textbf{86.77\tiny(0.42)}          & 60.11\tiny(1.33)          & 64.12\tiny(0.32)          & 55.28\tiny(4.71)          & 66.57\tiny(1.36)          \\
		
		&	$\text{RC}_{\text{img}1\mhyphen7,p=0.5,\lambda=10}$  & 84.48\tiny(0.52)          & 62.61\tiny(1.23)          & 66.13\tiny(0.80)          & 69.24\tiny(0.80)          & 70.61\tiny(0.53)          \\
		
		&	$\text{RC}_{\text{mix}1\mhyphen7, \lambda=10}$    & 85.59\tiny(0.40)          & 63.30\tiny(0.99)          & 63.83\tiny(0.85)          & 68.29\tiny(1.27)          & 70.25\tiny(0.45)          \\
		& {SIN} & 85.33\tiny(0.66)	& \textbf{65.85\tiny(0.87)}&	65.39\tiny(0.62)	&65.75\tiny(0.59)&	70.58\tiny(0.21) \\
		
		\midrule
	
		\multirow{3}{*}{\shortstack{$\text{RC}_{\text{img}1\mhyphen7,}$ \\ \tiny $p$=0.5,$\lambda$=10}}   & Baseline     & 81.15\tiny(0.76)          & 59.56\tiny(0.79)          & 62.42\tiny(0.59)          & 71.74\tiny(0.43)          & 68.72\tiny(0.58)          \\

		&$\text{RC}_{\text{img}1\mhyphen7,p=0.5,\lambda=10}$   & 84.36\tiny(0.36)          & { 63.73\tiny(0.91)}      & \textbf{68.07\tiny(0.55)} & {\ul75.41\tiny(0.57)}          & {\ul72.89\tiny(0.33)}          \\
		&$\text{RC}_{\text{mix}1\mhyphen7, \lambda=10}$       & 84.63\tiny(0.97)          & 63.41\tiny(1.22)          & 66.36\tiny(0.43)          & 74.59\tiny(0.84)          & 72.25\tiny(0.54)          \\
		\midrule
		\multirow{3}{*}{\shortstack{$\text{RC}_{\text{mix}1\mhyphen7}$ \\\tiny $\lambda$=10 }}      & Baseline     & 81.78\tiny(1.11)          & 61.14\tiny(0.51)          & 63.57\tiny(0.29)          & 71.97\tiny(0.38)          & 69.62\tiny(0.24)          \\
		&$\text{RC}_{\tiny\text{img}1\mhyphen7,p=0.5,\lambda=10}$  & {85.16\tiny(1.03)}    & 63.17\tiny(0.38)          & {\ul67.68\tiny(0.60)}        & \textbf{76.11\tiny(0.43)}    & \textbf{73.03\tiny(0.46)}    \\
		&$\text{RC}_{\text{mix}1\mhyphen7, \lambda=10}$   & {\ul86.17\tiny(0.56)} & {\ul65.33\tiny(1.05)} & 65.52\tiny(1.13)          & 73.21\tiny(1.03)          & 72.56\tiny(0.50)         \\
		\bottomrule
	\end{tabular}
	\vspace{-5mm}
\end{table}

\textbf{Better Performance via RandConv pretrained model}  We start by testing the Deep-All baselines using the two {\RandConv}-trained ImageNet models of Sec.~\ref{section:imagenet-sketch} as initialization. Tab.~\ref{table:transfer_PACS} shows significant improvements on Sketch. Results are comparable to finetuning with {\RandConv} on a normal pretrained model. Art is also consistently improved. Performance drops slightly on Photo as expected, since we reduced the texture bias in the pretrained model, which is helpful for the Photo domain. {A similar performance improvement is observed when using the SIN-trained AlexNet as initialization.} 
Using {\RandConv} for \emph{both} ImageNet training and PACS finetuning, we achieve 76.11\% accuracy on Sketch. As far as we know, this is the best performance using an AlexNet baseline. This approach even outperforms Jigen~\citep{carlucci2019jigen} (71.35\%) with a stronger ResNet18 baseline model. Cartoon and Art are also improved. The best average domain generalization accuracy is 73.03\%, with a more than 6\% improvement over our initial Deep-All baseline. 

This experiment confirms that generalizability may transfer: removing texture bias may not only make a pretrained model more generalizable, but it may help generalization on downstream tasks. For similar target and pretraining domains like Photo and ImageNet, where learning texture bias may actually be beneficial, performance may degrade slightly.

\vspace{-1mm}
\section{Conclusion and Discussion}
Randomized convolution ({\RandConv}) is a simple but powerful data augmentation technique for randomizing local image texture. {\RandConv} helps focus visual representations on global shape information rather than local texture. We theoretically justified the approximate shape-preserving property of {\RandConv} and developed {\RandConv} techniques using multi-scale and mixing designs. We also make use of a consistency loss to encourage texture invariance. 
{\RandConv} outperforms state-of-the-art approaches on the digit recognition benchmark and on the sketch domain of PACS and on ImageNet-Sketch by a large margin. By finetuning a model pretrained with {\RandConv} on PACS, we showed that the generalizability of a pretrained model may transfer to and benefit a new downstream task. This resulted in a new state-of-art performance on PACS in the Sketch domain. 

{{\RandConv} can help computer vision tasks when a shape-biased model is helpful e.g. for object detection. {\RandConv} can also provide a shape-biased pretrained model to improve performance on downstream tasks when generalizing to unseen domains.}
However, local texture features can be useful for many computer vision tasks, especially for fixed-domain fine-grained visual recognition. In such cases, visual representations that are invariant to local texture 
may hurt in-domain performance. Therefore, important future work includes learning representations that disentangle shape and texture features and building models to use such representations in an explainable way. 

{Adversarial robustness of deep neural networks has received significant recent attention. Interestingly, \cite{zhang2019interpreting} find that adversarially-trained models are more shape biased; \cite{shi2020informative} show that their method for increasing shape bias also helps adversarial robustness, especially when combined with adversarial training. Therefore, exploring how {\RandConv} affects the adversarial robustness of models could be interesting future work. Moreover, recent biologically inspired models for improving adversarial robustness \citep{dapello2020simulating} use Gabor filters with fixed random configurations followed by a stochastic layer to add Gaussian noise to the network input, which may explain the importance of randomness in {\RandConv}. Exploring connections between {\RandConv} and biological mechanisms in the human visual system would be interesting future work. }

\textbf{Acknowledgments} We thank Zhiding Yu for discussions on initial ideas and the experimental setup. We also thank Nathan Cahill for advice on proving the properties of random convolutions.
\newpage



\bibliography{mybib}

\begin{thebibliography}{51}
\providecommand{\natexlab}[1]{#1}
\providecommand{\url}[1]{\texttt{#1}}
\expandafter\ifx\csname urlstyle\endcsname\relax
  \providecommand{\doi}[1]{doi: #1}\else
  \providecommand{\doi}{doi: \begingroup \urlstyle{rm}\Url}\fi

\bibitem[Balaji et~al.(2018)Balaji, Sankaranarayanan, and
  Chellappa]{balaji2018metareg}
Yogesh Balaji, Swami Sankaranarayanan, and Rama Chellappa.
\newblock Metareg: Towards domain generalization using meta-regularization.
\newblock In \emph{Advances in Neural Information Processing Systems}, pp.\
  998--1008, 2018.

\bibitem[Berthelot et~al.(2019)Berthelot, Carlini, Goodfellow, Papernot,
  Oliver, and Raffel]{berthelot2019mixmatch}
David Berthelot, Nicholas Carlini, Ian Goodfellow, Nicolas Papernot, Avital
  Oliver, and Colin~A Raffel.
\newblock Mixmatch: A holistic approach to semi-supervised learning.
\newblock In \emph{Advances in Neural Information Processing Systems}, pp.\
  5049--5059, 2019.

\bibitem[Burda et~al.(2019)Burda, Edwards, Storkey, and
  Klimov]{burda2018exploration}
Yuri Burda, Harrison Edwards, Amos Storkey, and Oleg Klimov.
\newblock Exploration by random network distillation.
\newblock In \emph{International Conference on Learning Representations}, 2019.
\newblock URL \url{https://openreview.net/forum?id=H1lJJnR5Ym}.

\bibitem[Carlucci et~al.(2019)Carlucci, D'Innocente, Bucci, Caputo, and
  Tommasi]{carlucci2019jigen}
Fabio~M Carlucci, Antonio D'Innocente, Silvia Bucci, Barbara Caputo, and
  Tatiana Tommasi.
\newblock Domain generalization by solving jigsaw puzzles.
\newblock In \emph{Proceedings of the IEEE Conference on Computer Vision and
  Pattern Recognition}, pp.\  2229--2238, 2019.

\bibitem[Dapello et~al.(2020)Dapello, Marques, Schrimpf, Geiger, Cox, and
  DiCarlo]{dapello2020simulating}
Joel Dapello, Tiago Marques, Martin Schrimpf, Franziska Geiger, David Cox, and
  James~J DiCarlo.
\newblock Simulating a primary visual cortex at the front of cnns improves
  robustness to image perturbations.
\newblock \emph{Advances in Neural Information Processing Systems}, 33, 2020.

\bibitem[Denker et~al.(1989)Denker, Gardner, Graf, Henderson, Howard, Hubbard,
  Jackel, Baird, and Guyon]{denker1989neural}
John~S Denker, WR~Gardner, Hans~Peter Graf, Donnie Henderson, Richard~E Howard,
  W~Hubbard, Lawrence~D Jackel, Henry~S Baird, and Isabelle Guyon.
\newblock Neural network recognizer for hand-written zip code digits.
\newblock In \emph{Advances in neural information processing systems}, pp.\
  323--331, 1989.

\bibitem[Gaier \& Ha(2019)Gaier and Ha]{gaier2019weight}
Adam Gaier and David Ha.
\newblock Weight agnostic neural networks.
\newblock In \emph{Advances in Neural Information Processing Systems}, pp.\
  5364--5378, 2019.

\bibitem[Ganin \& Lempitsky(2014)Ganin and Lempitsky]{ganin2014unsupervised}
Yaroslav Ganin and Victor Lempitsky.
\newblock Unsupervised domain adaptation by backpropagation.
\newblock \emph{arXiv preprint arXiv:1409.7495}, 2014.

\bibitem[Ganin et~al.(2016)Ganin, Ustinova, Ajakan, Germain, Larochelle,
  Laviolette, Marchand, and Lempitsky]{ganin2016domain}
Yaroslav Ganin, Evgeniya Ustinova, Hana Ajakan, Pascal Germain, Hugo
  Larochelle, Fran{\c{c}}ois Laviolette, Mario Marchand, and Victor Lempitsky.
\newblock Domain-adversarial training of neural networks.
\newblock \emph{The Journal of Machine Learning Research}, 17\penalty0
  (1):\penalty0 2096--2030, 2016.

\bibitem[Geirhos et~al.(2019)Geirhos, Rubisch, Michaelis, Bethge, Wichmann, and
  Brendel]{geirhos2018imagenettrained}
Robert Geirhos, Patricia Rubisch, Claudio Michaelis, Matthias Bethge, Felix~A.
  Wichmann, and Wieland Brendel.
\newblock Imagenet-trained {CNN}s are biased towards texture; increasing shape
  bias improves accuracy and robustness.
\newblock In \emph{International Conference on Learning Representations}, 2019.
\newblock URL \url{https://openreview.net/forum?id=Bygh9j09KX}.

\bibitem[Ghifary et~al.(2016)Ghifary, Balduzzi, Kleijn, and
  Zhang]{ghifary2016scatter}
Muhammad Ghifary, David Balduzzi, W~Bastiaan Kleijn, and Mengjie Zhang.
\newblock Scatter component analysis: A unified framework for domain adaptation
  and domain generalization.
\newblock \emph{IEEE transactions on pattern analysis and machine
  intelligence}, 39\penalty0 (7):\penalty0 1414--1430, 2016.

\bibitem[Gulrajani \& Lopez-Paz(2020)Gulrajani and
  Lopez-Paz]{gulrajani2020search}
Ishaan Gulrajani and David Lopez-Paz.
\newblock In search of lost domain generalization.
\newblock \emph{arXiv preprint arXiv:2007.01434}, 2020.

\bibitem[He et~al.(2015)He, Zhang, Ren, and Sun]{he2015delving}
Kaiming He, Xiangyu Zhang, Shaoqing Ren, and Jian Sun.
\newblock Delving deep into rectifiers: Surpassing human-level performance on
  imagenet classification.
\newblock In \emph{Proceedings of the IEEE international conference on computer
  vision}, pp.\  1026--1034, 2015.

\bibitem[He et~al.(2016{\natexlab{a}})He, Zhang, Ren, and Sun]{he2016deep}
Kaiming He, Xiangyu Zhang, Shaoqing Ren, and Jian Sun.
\newblock Deep residual learning for image recognition.
\newblock In \emph{Proceedings of the IEEE conference on computer vision and
  pattern recognition}, pp.\  770--778, 2016{\natexlab{a}}.

\bibitem[He et~al.(2016{\natexlab{b}})He, Wang, and Hopcroft]{he2016powerful}
Kun He, Yan Wang, and John Hopcroft.
\newblock A powerful generative model using random weights for the deep image
  representation.
\newblock In \emph{Advances in Neural Information Processing Systems}, pp.\
  631--639, 2016{\natexlab{b}}.

\bibitem[Heeger \& Bergen(1995)Heeger and Bergen]{heeger1995pyramid}
David~J Heeger and James~R Bergen.
\newblock Pyramid-based texture analysis/synthesis.
\newblock In \emph{Proceedings of the 22nd annual conference on Computer
  graphics and interactive techniques}, pp.\  229--238, 1995.

\bibitem[Hendrycks et~al.(2020{\natexlab{a}})Hendrycks, Basart, Mu, Kadavath,
  Wang, Dorundo, Desai, Zhu, Parajuli, Guo, et~al.]{hendrycks2020many}
Dan Hendrycks, Steven Basart, Norman Mu, Saurav Kadavath, Frank Wang, Evan
  Dorundo, Rahul Desai, Tyler Zhu, Samyak Parajuli, Mike Guo, et~al.
\newblock The many faces of robustness: A critical analysis of
  out-of-distribution generalization.
\newblock \emph{arXiv preprint arXiv:2006.16241}, 2020{\natexlab{a}}.

\bibitem[Hendrycks et~al.(2020{\natexlab{b}})Hendrycks, Mu, Cubuk, Zoph,
  Gilmer, and Lakshminarayanan]{hendrycks2020augmix}
Dan Hendrycks, Norman Mu, Ekin~Dogus Cubuk, Barret Zoph, Justin Gilmer, and
  Balaji Lakshminarayanan.
\newblock Augmix: A simple method to improve robustness and uncertainty under
  data shift.
\newblock In \emph{International Conference on Learning Representations},
  2020{\natexlab{b}}.
\newblock URL \url{https://openreview.net/forum?id=S1gmrxHFvB}.

\bibitem[Ilyas et~al.(2019)Ilyas, Santurkar, Tsipras, Engstrom, Tran, and
  Madry]{ilyas2019adversarial}
Andrew Ilyas, Shibani Santurkar, Dimitris Tsipras, Logan Engstrom, Brandon
  Tran, and Aleksander Madry.
\newblock Adversarial examples are not bugs, they are features.
\newblock In \emph{Advances in Neural Information Processing Systems}, pp.\
  125--136, 2019.

\bibitem[Johnson \& Lindenstrauss(1984)Johnson and
  Lindenstrauss]{johnson1984extensions}
William~B Johnson and Joram Lindenstrauss.
\newblock Extensions of lipschitz mappings into a hilbert space.
\newblock \emph{Contemporary mathematics}, 26\penalty0 (189-206):\penalty0 1,
  1984.

\bibitem[Kornblith et~al.(2019)Kornblith, Shlens, and Le]{kornblith2019better}
Simon Kornblith, Jonathon Shlens, and Quoc~V Le.
\newblock Do better imagenet models transfer better?
\newblock In \emph{Proceedings of the IEEE conference on computer vision and
  pattern recognition}, pp.\  2661--2671, 2019.

\bibitem[LeCun et~al.(1998)LeCun, Bottou, Bengio, and
  Haffner]{lecun1998gradient}
Yann LeCun, L{\'e}on Bottou, Yoshua Bengio, and Patrick Haffner.
\newblock Gradient-based learning applied to document recognition.
\newblock \emph{Proceedings of the IEEE}, 86\penalty0 (11):\penalty0
  2278--2324, 1998.

\bibitem[Lee et~al.(2020)Lee, Lee, Shin, and Lee]{lee2020network}
Kimin Lee, Kibok Lee, Jinwoo Shin, and Honglak Lee.
\newblock Network randomization: A simple technique for generalization in deep
  reinforcement learning.
\newblock In \emph{International Conference on Learning Representations.
  https://openreview. net/forum}, 2020.

\bibitem[Li et~al.(2017)Li, Yang, Song, and Hospedales]{li2017deeper}
Da~Li, Yongxin Yang, Yi-Zhe Song, and Timothy~M Hospedales.
\newblock Deeper, broader and artier domain generalization.
\newblock In \emph{Proceedings of the IEEE international conference on computer
  vision}, pp.\  5542--5550, 2017.

\bibitem[Li et~al.(2018{\natexlab{a}})Li, Yang, Song, and
  Hospedales]{li2018mldg}
Da~Li, Yongxin Yang, Yi-Zhe Song, and Timothy~M Hospedales.
\newblock Learning to generalize: Meta-learning for domain generalization.
\newblock In \emph{Thirty-Second AAAI Conference on Artificial Intelligence},
  2018{\natexlab{a}}.

\bibitem[Li et~al.(2018{\natexlab{b}})Li, Jialin~Pan, Wang, and
  Kot]{li2018domain}
Haoliang Li, Sinno Jialin~Pan, Shiqi Wang, and Alex~C Kot.
\newblock Domain generalization with adversarial feature learning.
\newblock In \emph{Proceedings of the IEEE Conference on Computer Vision and
  Pattern Recognition}, pp.\  5400--5409, 2018{\natexlab{b}}.

\bibitem[Li et~al.(2018{\natexlab{c}})Li, Tian, Gong, Liu, Liu, Zhang, and
  Tao]{li2018ciddg}
Ya~Li, Xinmei Tian, Mingming Gong, Yajing Liu, Tongliang Liu, Kun Zhang, and
  Dacheng Tao.
\newblock Deep domain generalization via conditional invariant adversarial
  networks.
\newblock In \emph{Proceedings of the European Conference on Computer Vision
  (ECCV)}, pp.\  624--639, 2018{\natexlab{c}}.

\bibitem[Luo et~al.(2019)Luo, Zheng, Guan, Yu, and Yang]{luo2019taking}
Yawei Luo, Liang Zheng, Tao Guan, Junqing Yu, and Yi~Yang.
\newblock Taking a closer look at domain shift: Category-level adversaries for
  semantics consistent domain adaptation.
\newblock In \emph{Proceedings of the IEEE Conference on Computer Vision and
  Pattern Recognition}, pp.\  2507--2516, 2019.

\bibitem[Mu \& Gilmer(2019)Mu and Gilmer]{mu2019mnist}
Norman Mu and Justin Gilmer.
\newblock Mnist-c: A robustness benchmark for computer vision.
\newblock \emph{arXiv preprint arXiv:1906.02337}, 2019.

\bibitem[Netzer et~al.(2011)Netzer, Wang, Coates, Bissacco, Wu, and
  Ng]{netzer2011reading}
Yuval Netzer, Tao Wang, Adam Coates, Alessandro Bissacco, Bo~Wu, and Andrew~Y
  Ng.
\newblock Reading digits in natural images with unsupervised feature learning.
\newblock \emph{NIPS Workshop on Deep Learning and Unsupervised Feature
  Learning}, 2011.

\bibitem[Osband et~al.(2018)Osband, Aslanides, and
  Cassirer]{osband2018randomized}
Ian Osband, John Aslanides, and Albin Cassirer.
\newblock Randomized prior functions for deep reinforcement learning.
\newblock In \emph{Advances in Neural Information Processing Systems}, pp.\
  8617--8629, 2018.

\bibitem[Peng et~al.(2019{\natexlab{a}})Peng, Bai, Xia, Huang, Saenko, and
  Wang]{peng2019moment}
Xingchao Peng, Qinxun Bai, Xide Xia, Zijun Huang, Kate Saenko, and Bo~Wang.
\newblock Moment matching for multi-source domain adaptation.
\newblock In \emph{Proceedings of the IEEE International Conference on Computer
  Vision}, pp.\  1406--1415, 2019{\natexlab{a}}.

\bibitem[Peng et~al.(2019{\natexlab{b}})Peng, Huang, Sun, and
  Saenko]{Peng2019DomainAL}
Xingchao Peng, Zijun Huang, Ximeng Sun, and Kate Saenko.
\newblock Domain agnostic learning with disentangled representations.
\newblock In \emph{ICML}, 2019{\natexlab{b}}.

\bibitem[Portilla \& Simoncelli(2000)Portilla and
  Simoncelli]{portilla2000parametric}
Javier Portilla and Eero~P Simoncelli.
\newblock A parametric texture model based on joint statistics of complex
  wavelet coefficients.
\newblock \emph{International journal of computer vision}, 40\penalty0
  (1):\penalty0 49--70, 2000.

\bibitem[Qiao et~al.(2020)Qiao, Zhao, and Peng]{qiao2020learning}
Fengchun Qiao, Long Zhao, and Xi~Peng.
\newblock Learning to learn single domain generalization.
\newblock \emph{arXiv preprint arXiv:2003.13216}, 2020.

\bibitem[Raghupathi \& Raghupathi(2014)Raghupathi and
  Raghupathi]{raghupathi2014big}
Wullianallur Raghupathi and Viju Raghupathi.
\newblock Big data analytics in healthcare: promise and potential.
\newblock \emph{Health information science and systems}, 2\penalty0
  (1):\penalty0 3, 2014.

\bibitem[Salman et~al.(2020)Salman, Ilyas, Engstrom, Kapoor, and
  Madry]{salman2020adversarially}
Hadi Salman, Andrew Ilyas, Logan Engstrom, Ashish Kapoor, and Aleksander Madry.
\newblock Do adversarially robust imagenet models transfer better?
\newblock \emph{arXiv preprint arXiv:2007.08489}, 2020.

\bibitem[Saxe et~al.(2011)Saxe, Koh, Chen, Bhand, Suresh, and
  Ng]{saxe2011random}
Andrew~M Saxe, Pang~Wei Koh, Zhenghao Chen, Maneesh Bhand, Bipin Suresh, and
  Andrew~Y Ng.
\newblock On random weights and unsupervised feature learning.
\newblock In \emph{ICML}, volume~2, pp.\ ~6, 2011.

\bibitem[Shao et~al.(2019)Shao, Lan, Li, and Yuen]{shao2019multi}
Rui Shao, Xiangyuan Lan, Jiawei Li, and Pong~C Yuen.
\newblock Multi-adversarial discriminative deep domain generalization for face
  presentation attack detection.
\newblock In \emph{Proceedings of the IEEE Conference on Computer Vision and
  Pattern Recognition}, pp.\  10023--10031, 2019.

\bibitem[Shen et~al.(2019)Shen, Xu, Zhu, Guibas, Fei-Fei, and
  Savarese]{shen2019situational}
William~B Shen, Danfei Xu, Yuke Zhu, Leonidas~J Guibas, Li~Fei-Fei, and Silvio
  Savarese.
\newblock Situational fusion of visual representation for visual navigation.
\newblock In \emph{Proceedings of the IEEE International Conference on Computer
  Vision}, pp.\  2881--2890, 2019.

\bibitem[Shi et~al.(2020)Shi, Zhang, Dai, Zhu, Mu, and
  Wang]{shi2020informative}
Baifeng Shi, Dinghuai Zhang, Qi~Dai, Zhanxing Zhu, Yadong Mu, and Jingdong
  Wang.
\newblock Informative dropout for robust representation learning: A shape-bias
  perspective.
\newblock In \emph{International Conference on Machine Learning}, 2020.

\bibitem[Simard et~al.(2003)Simard, Steinkraus, Platt, et~al.]{simard2003best}
Patrice~Y Simard, David Steinkraus, John~C Platt, et~al.
\newblock Best practices for convolutional neural networks applied to visual
  document analysis.
\newblock In \emph{Icdar}, volume~3, 2003.

\bibitem[Sun \& Saenko(2014)Sun and Saenko]{sun2014virtual}
Baochen Sun and Kate Saenko.
\newblock From virtual to reality: Fast adaptation of virtual object detectors
  to real domains.
\newblock In \emph{Proceedings of the British Machine Vision Conference}. BMVA
  Press, 2014.

\bibitem[Tobin et~al.(2017)Tobin, Fong, Ray, Schneider, Zaremba, and
  Abbeel]{tobin2017domain}
Josh Tobin, Rachel Fong, Alex Ray, Jonas Schneider, Wojciech Zaremba, and
  Pieter Abbeel.
\newblock Domain randomization for transferring deep neural networks from
  simulation to the real world.
\newblock In \emph{2017 IEEE/RSJ international conference on intelligent robots
  and systems (IROS)}, pp.\  23--30. IEEE, 2017.

\bibitem[Vinh et~al.(2016)Vinh, Erfani, Paisitkriangkrai, Bailey, Leckie, and
  Ramamohanarao]{vinh2016training}
Nguyen~Xuan Vinh, Sarah Erfani, Sakrapee Paisitkriangkrai, James Bailey,
  Christopher Leckie, and Kotagiri Ramamohanarao.
\newblock Training robust models using random projection.
\newblock In \emph{2016 23rd International Conference on Pattern Recognition
  (ICPR)}, pp.\  531--536. IEEE, 2016.

\bibitem[Volpi et~al.(2018)Volpi, Namkoong, Sener, Duchi, Murino, and
  Savarese]{volpi2018generalizing}
Riccardo Volpi, Hongseok Namkoong, Ozan Sener, John~C Duchi, Vittorio Murino,
  and Silvio Savarese.
\newblock Generalizing to unseen domains via adversarial data augmentation.
\newblock In \emph{Advances in Neural Information Processing Systems}, pp.\
  5334--5344, 2018.

\bibitem[Wang et~al.(2019{\natexlab{a}})Wang, Ge, Lipton, and
  Xing]{wang2019learning}
Haohan Wang, Songwei Ge, Zachary Lipton, and Eric~P Xing.
\newblock Learning robust global representations by penalizing local predictive
  power.
\newblock In \emph{Advances in Neural Information Processing Systems}, pp.\
  10506--10518, 2019{\natexlab{a}}.

\bibitem[Wang et~al.(2019{\natexlab{b}})Wang, He, and Xing]{wang2018learning}
Haohan Wang, Zexue He, and Eric~P. Xing.
\newblock Learning robust representations by projecting superficial statistics
  out.
\newblock In \emph{International Conference on Learning Representations},
  2019{\natexlab{b}}.
\newblock URL \url{https://openreview.net/forum?id=rJEjjoR9K7}.

\bibitem[Wieting \& Kiela(2019)Wieting and Kiela]{wieting2018no}
John Wieting and Douwe Kiela.
\newblock No training required: Exploring random encoders for sentence
  classification.
\newblock In \emph{International Conference on Learning Representations}, 2019.
\newblock URL \url{https://openreview.net/forum?id=BkgPajAcY7}.

\bibitem[Yue et~al.(2019)Yue, Zhang, Zhao, Sangiovanni-Vincentelli, Keutzer,
  and Gong]{yue2019domain}
Xiangyu Yue, Yang Zhang, Sicheng Zhao, Alberto Sangiovanni-Vincentelli, Kurt
  Keutzer, and Boqing Gong.
\newblock Domain randomization and pyramid consistency: Simulation-to-real
  generalization without accessing target domain data.
\newblock In \emph{Proceedings of the IEEE International Conference on Computer
  Vision}, pp.\  2100--2110, 2019.

\bibitem[Zhang \& Zhu(2019)Zhang and Zhu]{zhang2019interpreting}
Tianyuan Zhang and Zhanxing Zhu.
\newblock Interpreting adversarially trained convolutional neural networks.
\newblock In \emph{International Conference on Machine Learning}, pp.\
  7502--7511, 2019.

\end{thebibliography}
\bibliographystyle{iclr2021_conference}

\newpage
\appendix
This supplementary material provides additional details. Specifically, in Sec.~\ref{sec:shapes} and \ref{theorem_proof}, we discuss definitions of shapes and textures in images and justify why random convolution preserves global shapes and disrupts local texture formally by proving Theorem~\ref{theorem1}. This theorem shows that random linear projections are approximately distance preserving. We also discuss our simulation-based bound based on 80\% distance rescaling on real image data. Sec.~\ref{exp_detail} provides more experimental details for the different datasets. Sec.~\ref{resnet_exp} shows experimental results with a stronger backbone architecture and on a new benchmark ImageNet-R~\citep{hendrycks2020many}.  Sec.~\ref{results} provides more detailed results regarding hyperparameter selection and ablation studies. Lastly, Sec.~\ref{examples} shows example visualizations of {\RandConv} outputs and for its mixing variant.

\section{Shapes and Texture in Images}
\label{sec:shapes}
\begin{figure}[t]
    \centering
	\newcommand\cwidth{0.3\textwidth}
	\begin{tabular}{ccc}
		\includegraphics[width=\cwidth]{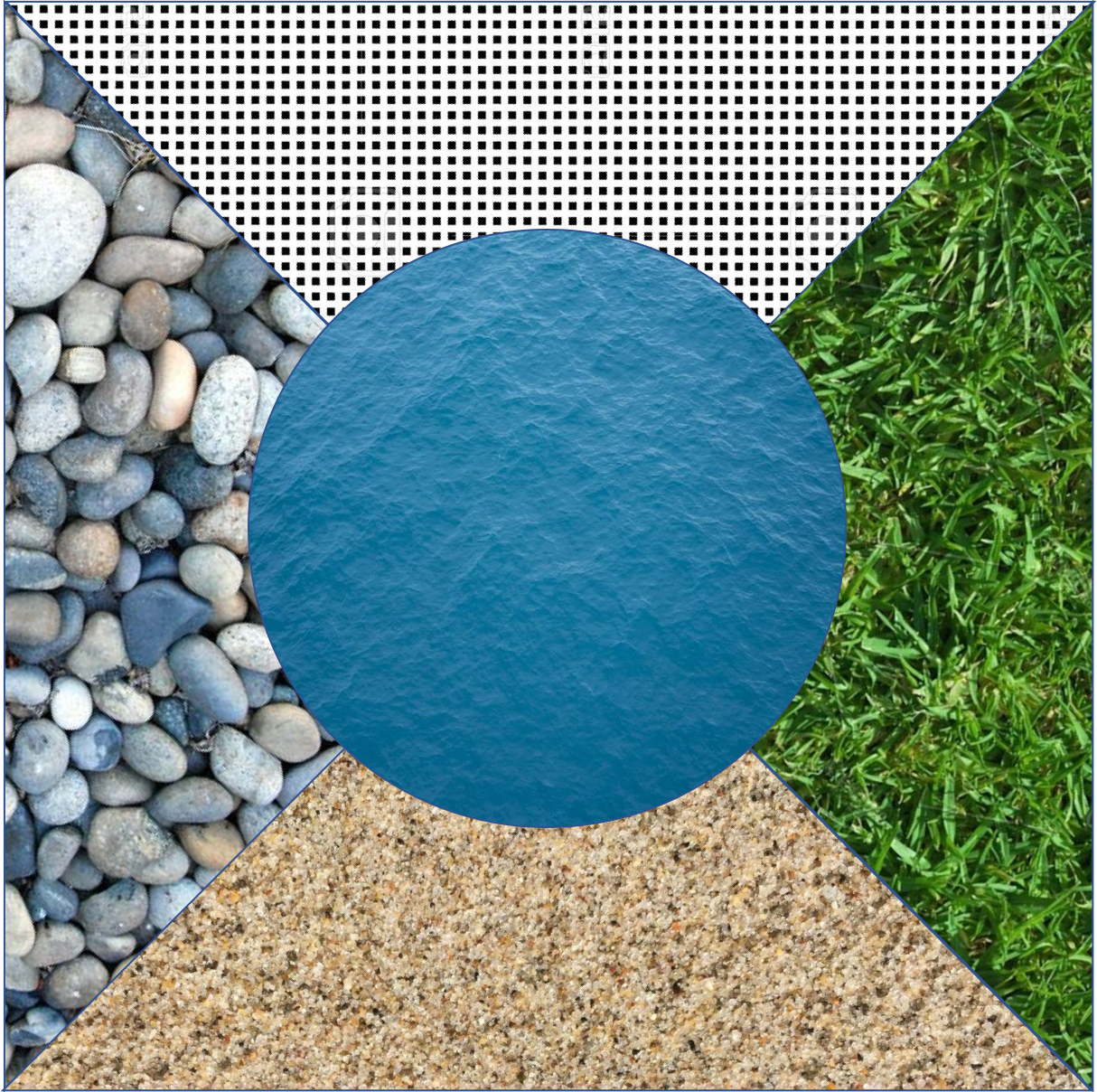} & \includegraphics[width=\cwidth]{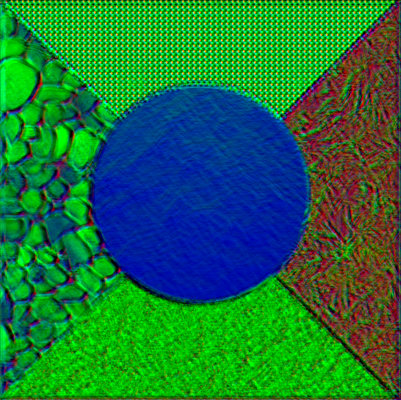}  & \includegraphics[width=\cwidth]{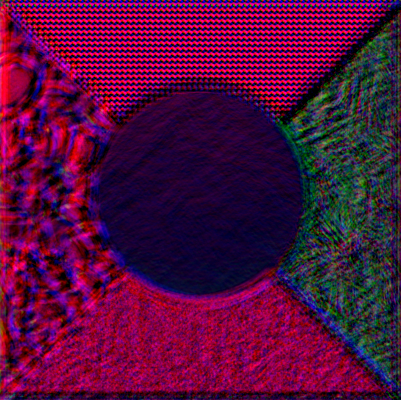} \\
	\end{tabular}
	\caption{\small \textbf{Left:} An image with texture and shapes at different scales; \textbf{Middle:} The output of {\RandConv} with a small filter size which largely preserves the shapes of the stones.  \textbf{Right:} The output of {\RandConv} with a large filter size distorts the shape of the stones as well.}
	\vspace{-5mm}
	\label{fig:texture_example}
\end{figure}
As discussed in the main text, we define shapes in images that are preserved by a random convolution layer as primitive shapes: spatial clusters of pixels with similar local texture. An object in a image can be a single primitive shape alone but in most cases it is the composition of multiple primitive shapes e.g. a car includes wheels, body frames, windshields. Note that the definition of texture is not necessarily opposite to shapes, since the texture of a larger shape can includes smaller shapes. For example, in Fig.\ref{fig:texture_example}, the left occluded triangle shape has texture composed by shapes of cobble stones while cobble stones have their own texture. Random convolution can preserve those large shapes that usually define the image semantics while distorting the small shapes as local texture.   

To formally define the shape-preserving property, we assume $(x_1, y_1)$, $(x_2, y_2)$ and $(x_3, y_3)$ are three locations on a image and $(x_1, y_1)$ has closer color and local texture with $(x_2, y_2)$ than $(x_3, y_3)$. For example, $(x_1, y_1)$ and $(x_2, y_2)$ are within the same shape while $(x_3, y_3)$ is located at a neighboring shape. Then we have $\|\mathbf{p}(x_1, y_1) - \mathbf{p}(x_2, y_2)\| < \|\mathbf{p}(x_1, y_1) - \mathbf{p}(x_3, y_3)\|$, where $\mathbf{p}(x_i, y_i)$ is the image patch at location $(x_i, y_i)$. A transformation $f$ is \emph{shape-preserving} if it \emph{maintains} such relative distance relations for most location triplets, i.e.
\begin{align}
\label{eq:dist_preserve}
\|f(\mathbf{p}(x_i, y_i)) - f(\mathbf{p}(x_j, y_j))\| / \|\mathbf{p}(x_i, y_i) - \mathbf{p}(x_j, y_j)\| \approx r \end{align}
for any two spatial location $(x_i, y_i)$ and $(x_j, y_j)$;  $r\geq 0$ is a constant. 

\section{Random Convolution is Shape-preserving as Random Linear Projection is Distance Preserving}
\label{theorem_proof}

We can express a convolution layer as a local linear projection:
\begin{align}
\label{eq:local_projection}
\mathbf{g}(x, y) & = \mathbf{U}\mathbf{p}(x,y)\,,
\end{align}
where $\mathbf{p}(x,y)\in\mathbb{R}^{d}$ ($d = h \times w \times C_{in}$) is the vectorized image patch centerized at location $(x,y)$, $\mathbf{g}(x,y)\in\mathbb{R}^{C_{out}}$ is the output feature at location $(x,y)$, and $\mathbf{U}\in\mathbb{R}^{C_{out}\times d }$ is the matrix expressing the convolution layer filters $\mathbf{\Theta}$. I.e., for each sliding window centered at $(x, y)$, a convolution layer applies a linear transform $f: \mathbb{R}^{d}\rightarrow\mathbb{R}^{C_{out}} $ projecting the $d$ dimensional local image patch $\mathbf{p}(x,y)$ to its $C_{out}$ dimensional feature $\mathbf{g}(x, y)$. When $\mathbf{\Theta}$ is independently randomly sampled, e.g. from a Gaussian distribution, the convolution layer preserves global shapes since that a random linear projection is \emph{approximately} distance-preserving by bounding the range of $r$ in Eq.~\ref{eq:dist_preserve} in Theorem~\ref{theorem1}.

\begin{theorem}
	\label{theorem1}
	Suppose we have N data points $\mathbf{z}_1,\cdots, \mathbf{z}_N \in \mathbb{R}^d$. Let $f(\mathbf{z}) = \mathbf{U} \mathbf{z}$ be a random linear projection $f: \mathbb{R}^{d}\rightarrow\mathbb{R}^{m} $ such that $\mathbf{U} \in \mathbb{R}^{m \times d}$ and $\mathbf{U}_{i,j} \sim N(0,\sigma^2)$. Then we have:
	\begin{equation}
	\begin{aligned}
	P\Big(\sup_{i\neq j; i,j  \in [N]} \Big\{r_{i,j} :=\frac{\norms{f(\mathbf{z}_i) - f(\mathbf{z}_j)}}{\norms{\mathbf{z}_i - \mathbf{z}_j}} \Big\} > \delta_1 \Big) \leq \epsilon,  \\\vspace{5ex}
	P\Big(\inf_{i\neq j; i,j  \in [N]} \Big\{r_{i,j} :=\frac{\norms{f(\mathbf{z}_i) - f(\mathbf{z}_j)}}{\norms{\mathbf{z}_i - \mathbf{z}_j}} \Big\} < \delta_2 \Big) \leq \epsilon,  
	\end{aligned}
	\end{equation}
	where $\delta_1 := \sigma\sqrt{\chi^2_{\frac{2\epsilon}{N(N-1)}}(m)}$ and $\delta_2 := \sigma\sqrt{\chi^2_{1 - \frac{2\epsilon}{N(N-1)}}(m)}$. Here, $\chi^2_{\alpha}(m)$ denotes the $\alpha$-upper quantile of the $\chi^2$ distribution with $m$ degrees of freedom.
	
\end{theorem}

Thm.~\ref{theorem1} tells us that for any data pair $(\mathbf{z}_i, \mathbf{z}_j)$ in a set of $N$ points, the distance rescaling ratio $r_{i,j}$ after a random linear projection is bounded by $\delta_1$ and $\delta_2$ with probability  $1 -\epsilon$. A Smaller $N$ and a larger output dimension $m$ give better bounds. E.g., when $m=3$, $N=1,000$, $\sigma=1$ and $\epsilon=0.1$, $\delta_1 = 5.8$ and $\delta_2 = 0.01$. Thm.~\ref{theorem1} gives a theoretical bound for \emph{all} the $N(N-1)/2$ pairs. However, in practice, preserving distances for a majority of $N(N-1)/2$ pairs is sufficient. To empirically verify this, we test
the range of central $80\%$ of $\{r_{i,j}\}$ on real image data. Using the same $(m,N,\sigma,\epsilon)$, $80\%$ of the pairs lie in 
$[0.56,2.87]$, which is significantly better than the strict bound: $[0.01,5.8]$. A proof of the theorem and simulation details are given in the following.

\begin{proof} 
	Let $\mathbf{U}_{k}$ represent to the $k$-th row of $\mathbf{U}$. It is easy to check that $\mathbf{v}_k := \iprods{\mathbf{U}_k, \mathbf{z}_i - \mathbf{z}_j}/\norms{\mathbf{z}_i - \mathbf{z}_j} \sim N(0, \sigma^2)$. Therefore,
	
	$$
	\frac{\norms{f(\mathbf{z}_i) - f(\mathbf{z}_j)}^2}{\sigma^2\norms{\mathbf{z}_i - \mathbf{z}_j}^2} = \frac{1}{\sigma^2 } \frac{(\mathbf{z}_i-\mathbf{z}_j)^{\top}\mathbf{U}^{\top}\mathbf{U}(\mathbf{z}_i-\mathbf{z}_j)}{\norms{\mathbf{z}_i - \mathbf{z}_j}^2} = \sum_{k = 1}^{ m} \frac{\mathbf{v}_k^2}{\sigma^2} \sim \chi^2(m).
	$$
	Therefore, for $0 < \epsilon < 1$, we have
	$$
	P\Big(\frac{\norms{f(\mathbf{z}_i) - f(\mathbf{z}_j)}^2}{\sigma^2\norms{\mathbf{z}_i - \mathbf{z}_j}^2} > \chi^2_{\frac{2\epsilon}{N(N-1)}}(m) \Big) \leq \frac{2\epsilon}{N(N-1)}.
	$$
	From the above inequality, we have
	$$
	\begin{array}{l}
	P\Big(\sup_{i\neq j; i,j  \in [N]}\Big\{ \frac{\norms{f(\mathbf{z}_i) - f(\mathbf{z}_j)}^2}{\norms{\mathbf{z}_i - \mathbf{z}_j}^2}\Big\} > \sigma^2\chi^2_{\frac{2\epsilon}{N(N-1)}}(m)\Big) \vspace{1ex}\\
	= P\Big(\sup_{i\neq j; i,j  \in [N]}\Big\{\frac{\norms{f(\mathbf{z}_i) - f(\mathbf{z}_j)}^2}{\sigma^2\norms{\mathbf{z}_i - \mathbf{z}_j}^2}\Big\} > \chi^2_{\frac{2\epsilon}{N(N-1)}}(m)\Big) \vspace{1ex}\\
	= P\Big(\bigcup\limits_{i\neq j; i,j  \in [N]}\Big\{\frac{\norms{f(\mathbf{z}_i) - f(\mathbf{z}_j)}^2}{\sigma^2\norms{\mathbf{z}_i - \mathbf{z}_j}^2} > \chi^2_{\frac{2\epsilon}{N(N-1)}}(m) \Big\}\Big) \vspace{1ex}\\
	\leq \sum\limits_{i\neq j; i,j  \in [N]} P\Big(\frac{\norms{f(\mathbf{z}_i) - f(\mathbf{z}_j)}^2}{\sigma^2\norms{\mathbf{z}_i - \mathbf{z}_j}^2} > \chi^2_{\frac{2\epsilon}{N(N-1)}}(m) \Big) \vspace{1ex}\\
	\leq \epsilon, \vspace{1ex}
	\end{array}
	$$
	
	which is equivalent to
	$$
	P\Big(\sup_{i\neq j; i,j  \in [N]}\Big\{\frac{\norms{f(\mathbf{z}_i) - f(\mathbf{z}_j)}}{\norms{\mathbf{z}_i - \mathbf{z}_j}}\Big\} > \sigma\sqrt{\chi^2_{\frac{2\epsilon}{N(N-1)}}(m)}\Big) \leq \epsilon.
	$$
	Similarly, we have
	$$
	P\Big(\inf_{i\neq j; i,j  \in [N]} \Big\{\frac{\norms{f(\mathbf{z}_i) - f(\mathbf{z}_j)}}{\norms{\mathbf{z}_i - \mathbf{z}_j}} \Big\} < \sigma\sqrt{\chi^2_{1 - \frac{2\epsilon}{N(N-1)}}(m)} \Big) \leq \epsilon.
	$$
	
\end{proof}

\textbf{Simulation on Real Image Data} To better understand the relative distance preservation property of random linear projections in practice, we use Algorithm~\ref{algo:simulation} to empirically obtain a bound for real image data. We choose $m=3$, $N=1,000$, $\sigma=1$ and $\epsilon=0.1$ as in computing our theoretical bounds. We use  $M=1,000$ real images from the PACS dataset for this simulation. Note that the image patch size or $d$ does not affect the bound. We use a patch size of $3\times3$ resulting in $d=27$. This simulation tell us that applying linear projections with a randomly sampled $U$ on $N$ local images patches in every image, we have a $1-\epsilon$ chance that $80\%$ of $r_{i,j}$ is in the range $[\delta_{10\%}, \delta_{90\%}]$. 		

\begin{algorithm}[h]
	\small
	\caption{Simulate the range of central 80\% of ${r_{i,j}}$ on real image data}
	\label{algo:simulation}
	\begin{algorithmic}[1]
		\State \textbf{Input}: $M$ images $\{I_i\}_{i=1}^{M}$, number of data points $N$, projection output dimension $m$, standard deviation $\sigma$ of normal distribution, confidence level $\epsilon$.
		
		\For {$m = 1 \to M$}
		\State Sample images patches in $I_m$ at 1,000 locations and vectorize them as $\{\mathbf{z}_l^m\}_{l=1}^{N}$
		\State Sample a projection matrix $\mathbf{U} \in \mathbb{R}^{m \times d}$ and $\mathbf{U}_{i,j} \sim N(0,\sigma^2)$
		\For {$i = 1 \to N$}
		\For {$j = i+1\to N$}
		\State Compute $r_{i,j}^m = \frac{\norms{f(\mathbf{z}_i^m) - f(\mathbf{z}_j^m)}}{\norms{\mathbf{z}_i^m - \mathbf{z}_j^m}}$, where $f(\mathbf{z}) = \mathbf{U} \mathbf{z}$
		\EndFor	
		\EndFor
		\State $q^m_{10\%}$ = $10\%$ quantile of $r_{i,j}^m$ for $I_m$ 
		\State $q^m_{90\%}$	= $90\%$ quantile of $r_{i,j}^m$ for $I_m$ \Comment{Get the central 80\% of $r_{i,j}$ in each image}
		\EndFor	
		\State $\delta_{10\%}$ = $\epsilon$ quantile of all $q^m_{10\%}$ 
		\State $\delta_{90\%}$ = $(1-\epsilon)$ quantile of all $q^m_{90\%}$
		\Comment{Get the $\epsilon$ confident bound for $q^m_{10\%}$ and $q^m_{90\%}$}
		\State \textbf{return} $\delta_{10\%}$, $\delta_{90\%}$		 
	\end{algorithmic}
\end{algorithm}

\section{Experimental Details}
\label{exp_detail}

\textbf{Digits Recognition} The network for our digits recognition experiments is composed of two \emph{Conv5$\times$5-ReLU-MaxPool2$\times$2} blocks with 64/128 output channels and three fully connected layer with 1024/1024/10 output channels. We train the network with batch size 32 for 10,000 iterations. During training, the model is validated every 250 iterations and saved with the best validation score for testing. We apply the \texttt{Adam} optimizer with an initial learning rate of 0.0001.   

\textbf{PACS} We use the official data splits for training/validation/testing; no extra data augmentation is applied. We use the official \texttt{PyTorch} implementation and the pretrained weights of AlexNet for our PACS experiments. AlextNet is finetuned for 50,000 iterations with a batch size 128. Samples are randomly selected from the training data mixed between the three domains. We use the validation data of  source domains only at every 100 iterations. We use the \texttt{SGD} optimizer for training with an initial learning rate of 0.001, Nesterov momentum, and weight decay set to 0.0005. We let the learning rate decay by a factor of 0.1 after finishing 80\% of the iterations.

\textbf{ImageNet} Following the \texttt{PyTorch} example \footnote{https://github.com/pytorch/examples/tree/master/imagenet} on training ImageNet models, we set the batch size to 256 and train AlexNet from scratch for 90 epochs. We apply the \texttt{SGD} optimizer with an initial learning rate of 0.01, momentum 0.9, and weight decay 0.0001. We reduce the learning rate via a factor of 0.1 every 30 epochs.

\section{More Experiments with ResNet-18}
\label{resnet_exp}
In this section, we demonstrate that {\RandConv} also works on other stronger backbone architectures, e.g. for a Residual Network~\citet{he2016deep}. Specifically, we run the PACS and ImageNet experiments with ResNet-18 as the baseline and {\RandConv}. As Table~\ref{table:imagenet_resnet} shows,  {\RandConv} improves the baseline using ResNet18 on ImageNet-sketch by 10.5\% accuracy. When using a {\RandConv} pretrained ResNet-18 on PACS, the performance of finetuning with DeepAll and {\RandConv} are both improved shown in Table~\ref{table:transfer_PACS_resnet}. The best average domain generalization accuracy is 84.09\%, with a more than 8\% improvement over our initial Deep-All baseline. A model pretrained with $\text{RC}_{\text{mix}{1\mhyphen7, \lambda=10}}$ generally performs better than when pretrained with $\text{RC}_{\text{img}1\mhyphen7,p=0.5,\lambda=10}$. We also provide the ResNet-18 performance of JiGen~\citep{carlucci2019jigen} on PACS as reference. Note that JiGen uses extra data augmentation and a different data split than our approach and it only improves over its own baseline by 1.5\%. In addition, we test {\RandConv} trained ResNet-18 on ImageNet-R \citep{hendrycks2020many}, a domain generalization benchmark that contains images of artistic renditions of 200 object classes from the original ImageNet dataset. As Table \ref{table:imagenetR_resnet} shows, {\RandConv} also improve the generalization performance on ImageNet-R and reduce the gap between the in-domain (ImageNet-200) and out-of-domain (ImageNet-R) performance. 

\begin{table}[htp]
	\small
	\caption{Accuracy of ImageNet-trained ResNet-18 on ImageNet-Sketch data.}
	\label{table:imagenet_resnet}
	\centering
	\begin{tabular}{c|cccc}
		\toprule
		      & Baseline &
		$\text{RC}_{\text{img}{1\mhyphen7}}$\tiny, $p$=0.5, $\lambda$=10   & $\text{RC}_{\text{mix}{1\mhyphen7}}$\tiny, $\lambda$=10   \\
		\toprule
		    Top1  & 20.23        & 
		28.79 &  30.70 \\
		 Top5 & 37.26 
		& 49.02 & 51.80 \\
		\bottomrule
	\end{tabular}
\end{table}

\begin{table}[htp]
	\small
	\caption{Top 1 Accuracy of ImageNet-trained ResNet-18 on ImageNet-R data. ImageNet-200 are the original ImageNet data with the same 200 classes as ImageNet-R. }
	\label{table:imagenetR_resnet}
	\centering
	\begin{tabular}{c|cccc}
		\toprule
		      & Baseline &
		$\text{RC}_{\text{img}{1\mhyphen7}}$\tiny, $p$=0.5, $\lambda$=10   & $\text{RC}_{\text{mix}{1\mhyphen7}}$\tiny, $\lambda$=10   \\
		\toprule
		    ImageNet-200 (\%)    &  88.15         & 
		83.72 &  72.7 \\
		 ImageNet-R (\%) &  33.06   
		& 37.38 & 35.75 \\
		Gap & 55.09 & 46.34 & 36.95\\
		\bottomrule
	\end{tabular}
\end{table}

\begin{table}[htp]
	\small
	\vspace{-2mm}
	\caption{Generalization results on PACS with {\RandConv} pretrained model using ResNet-18. ImageNet column shows how the pretrained model is trained on ImageNet (baseline represents training using only the classification loss); PACS column indicates the methods used for finetuning on PACS.  \textbf{Best} and \underline{second best} accuracy for each target domain are highlighted in bold and underlined. The performance of JiGen~\citep{carlucci2019jigen} and its baseline using ResNet-18 is also given. }
	\label{table:transfer_PACS_resnet}
	\centering
	\begin{tabular}{c|c|ccccc}
		\toprule
		PACS                & ImageNet & Photo & Art & Cartoon & Sketch          & Avg                  \\
		\toprule
		\multirow{3}{*}{Deep-All} & Baseline     & \textbf{95.45\tiny(0.43)} &        74.96\tiny(0.99) &        71.48\tiny(1.22) &        62.09\tiny(1.12) &        76.00\tiny(0.37)          \\
		&	$\text{RC}_{\text{img}1\mhyphen7,p=0.5,\lambda=10}$  & 94.65\tiny(0.16) &        73.85\tiny(0.97) &        74.78\tiny(0.58) &        73.51\tiny(1.16) &        79.20\tiny(0.40)   \\
		&	$\text{RC}_{\text{mix}1\mhyphen7, \lambda=10}$    & 94.10\tiny(0.43) &        76.72\tiny(1.43) &        73.41\tiny(1.29) &        77.60\tiny(0.55) &        80.46\tiny(0.74)         \\
		\midrule
	
		\multirow{3}{*}{\shortstack{$\text{RC}_{\text{img}1\mhyphen7,}$ \\ \tiny $p$=0.5,$\lambda$=10}}   & Baseline     & 92.37\tiny(0.54) &        76.50\tiny(0.55) &        71.33\tiny(0.29) &        79.65\tiny(1.32) &        79.96\tiny(0.53)              \\

		&$\text{RC}_{\text{img}1\mhyphen7,p=0.5,\lambda=10}$   & 94.43\tiny(0.22) &        79.80\tiny(1.03) &        73.40\tiny(0.37) &        81.51\tiny(0.85) &        82.28\tiny(0.38)         \\
		&$\text{RC}_{\text{mix}1\mhyphen7, \lambda=10}$       & 94.57\tiny(0.45) &        \textbf{81.32\tiny(1.00)} &        \textbf{76.28\tiny(0.82)} &        \textbf{84.18\tiny(0.94)} &        \textbf{84.09\tiny(0.61)}          \\
		\midrule
		\multirow{3}{*}{\shortstack{$\text{RC}_{\text{mix}1\mhyphen7}$ \\\tiny $\lambda$=10 }}      & Baseline     & 93.57\tiny(0.40) &        77.73\tiny(0.91) &        71.24\tiny(0.91) &        75.53\tiny(2.17) &        79.52\tiny(0.61)          \\
		&$\text{RC}_{\tiny\text{img}1\mhyphen7,p=0.5,\lambda=10}$  & {\ul95.23\tiny(0.30)} &        80.56\tiny(0.82) &        74.18\tiny(0.53) &        80.70\tiny(1.43) &        82.67\tiny(0.46)    \\
		&$\text{RC}_{\text{mix}1\mhyphen7, \lambda=10}$   & 95.01\tiny(0.32) &        {\ul81.09\tiny(1.24)} &        {\ul76.04\tiny(0.92)} &        {\ul83.02\tiny(0.93)} &        {\ul83.79\tiny(0.60)}         \\
		\midrule
		Deep-All & \multirow{2}{*}{Baseline} & 95.73 & 77.85 & 74.86 & 67.74  & 79.05 \\
		JiGen & &  96.03 & 79.42 & 75.25 & 71.35 & 80.51 \\
		\bottomrule
	\end{tabular}
	\vspace{-2mm}
\end{table}

\section{Hyperparameter Selections and Ablation Studies on Digits Recognition Benchmarks}

We provide detailed experimental results for the digits recognition datasets.  Table~\ref{digits-p} shows results for different hyperameters $p$ for $\text{RC}_{\text{img}{1}}$. Table~\ref{digits-scale} shows results for an ablation study on the multi-scale design for $\text{RC}_{\text{mix}}$ and $\text{RC}_{\text{img}, p=0.5}$. Table~\ref{digits-consistency} shows results for studying the consistency loss weight $\lambda$ for $\text{RC}_{\text{mix}1\mhyphen7}$ and $\text{RC}_{\text{img}1\mhyphen7, p=0.5}$. Tables~\ref{digits-p},~\ref{digits-scale},~and~\ref{digits-consistency} correspond to Fig. 2 (a)(b)(c) in the main text respectively.
\label{results}
\begin{table}[htp]
	\small
	\centering
	\caption{Ablation study of hyperparameter $p$ for $\text{RC}_{\text{img}{1}}$ on digits recognition benchmarks. DG-Avg is the average performance on MNIST-M, SVHN, SYNTH and USPS. Best results are \textbf{bold}.}
	\label{digits-p}
	\begin{tabular}{l|c|ccccc|c}
		\toprule
		& MNIST-10k    & MNIST-M    & SVHN        & USPS        & SYNTH       & DG Avg         & MNIST-C     \\
		\midrule
		Baseline    & 98.40\tiny(0.84) & 58.87\tiny(3.73) & 33.41\tiny(5.28) & 79.27\tiny(2.70) & 42.43\tiny(5.46) & 53.50\tiny(4.23) & 88.20\tiny(2.10) \\
		$\text{RC}_{\text{img}{1}}$\tiny, $p$=0.9 & 98.68\tiny(0.06) & 83.53\tiny(0.37) & 53.67\tiny(1.54) & 80.38\tiny(1.41) & 59.19\tiny(0.85) & 69.19\tiny(0.34) & \textbf{89.79\tiny(0.44)} \\
		$\text{RC}_{\text{img}{1}}$\tiny, $p$=0.7 & 98.64\tiny(0.07) & 84.17\tiny(0.61) & 54.50\tiny(1.55) & \textbf{80.85\tiny(0.91)} & 60.25\tiny(0.85) & 69.94\tiny(0.50) & 89.20\tiny(0.60) \\
		$\text{RC}_{\text{img}{1}}$\tiny, $p$=0.5 & 98.72\tiny(0.08) & 85.17\tiny(1.12) & \textbf{55.97\tiny(0.54)} & 80.31\tiny(0.85) & \textbf{61.07\tiny(0.47)} & \textbf{70.63\tiny(0.42)} & 88.66\tiny(0.62) \\
		$\text{RC}_{\text{img}{1}}$\tiny, $p$=0.3 & 98.71\tiny(0.12) & 85.45\tiny(0.87) & 54.62\tiny(1.52) & 79.78\tiny(1.40) & 60.51\tiny(0.41) & 70.09\tiny(0.60) & {89.02\tiny(0.32)} \\
		$\text{RC}_{\text{img}{1}}$\tiny, $p$=0.1 & 98.66\tiny(0.06) & 85.57\tiny(0.79) & 54.34\tiny(1.52) & 79.21\tiny(0.44) & 60.18\tiny(0.63) & 69.83\tiny(0.38) & 88.53\tiny(0.38) \\
		$\text{RC}_{\text{img}{1}}$\tiny, $p$=0 & 98.55\tiny(0.13) & \textbf{86.27\tiny(0.42)} & 52.48\tiny(3.00) & 79.01\tiny(1.11) & 59.53\tiny(1.14) & 69.32\tiny(1.19) & 88.01\tiny(0.36) \\
		\bottomrule
	\end{tabular}
\end{table}
\begin{table}[htp]
	\small
	\centering
	\setlength{\tabcolsep}{3pt}
	\caption{Ablation study of multi-scale {\RandConv} on digits recognition benchmarks for $\text{RC}_{\text{mix}}$ and $\text{RC}_{\text{img}, p=0.5}$. Best entries for each variant are \textbf{bold}.}
	\label{digits-scale}
	\begin{tabular}{l|c|ccccc|c}
		\toprule
		& MNIST-10k    & MNIST-M    & SVHN        & USPS        & SYNTH       & DG Avg         & MNIST-C     \\
		\midrule
		$\text{RC}_{\text{mix}{1}}$       & 98.62\tiny(0.06)          & 83.98\tiny(0.98)          & 53.26\tiny(2.59)          & 80.57\tiny(1.09)          & 59.25\tiny(1.38)          & 69.26\tiny(1.35)          & 88.59\tiny(0.38)          \\
		$\text{RC}_{\text{mix}{1\mhyphen3}}$  & 98.76\tiny(0.02)          & 84.66\tiny(1.67)          & 55.89\tiny(0.83)          & 80.95\tiny(1.15)          & 60.07\tiny(1.05)          & 70.39\tiny(0.58)          & 89.80\tiny(0.94)          \\
		$\text{RC}_{\text{mix}{1\mhyphen5}}$  & 98.76\tiny(0.06)          & 84.32\tiny(0.43)          & \textbf{56.50\tiny(2.68)} & 81.85\tiny(1.05)    & 60.76\tiny(1.02)          & 70.86\tiny(0.86)          & 90.06\tiny(0.80)          \\
		$\text{RC}_{\text{mix}{1\mhyphen7}}$  & 98.82\tiny(0.06)    	& 84.91\tiny(0.68)         & 55.61\tiny(2.63)   & \textbf{82.09\tiny(1.00)} & \textbf{62.15\tiny(1.30)}   & \textbf{71.19\tiny(1.21)}        & 90.30\tiny(0.44)   \\
		$\text{RC}_{\text{mix}{1\mhyphen9}}$ &  98.81\tiny(0.12) & \textbf{85.13\tiny(0.72)}          & 54.18\tiny(3.36)          & 82.07\tiny(1.28)          & 61.85\tiny(1.41)          & 70.81\tiny(1.24)          & \textbf{90.83\tiny(0.52)}   \\
		\midrule
		$\text{RC}_{\text{img}{1}}$\tiny, $p$=0.5   & 98.66\tiny(0.05) & 85.12\tiny(0.96)          & 55.59\tiny(0.29)          & 80.65\tiny(0.71)          & 60.85\tiny(0.48)          & 70.55\tiny(0.15)          & 89.00\tiny(0.45)          \\
		$\text{RC}_{\text{img}{1\mhyphen3}}$\tiny, $p$=0.5 & 98.79\tiny(0.07) & 85.36\tiny(1.04)          & \textbf{55.60\tiny(1.09)} & 80.99\tiny(0.99)          & 61.26\tiny(0.80)          & 70.80\tiny(0.86)          & 89.84\tiny(0.70)          \\
		$\text{RC}_{\text{img}{1\mhyphen5}}$\tiny, $p$=0.5 & 98.83\tiny(0.07) & \textbf{86.33\tiny(0.47)} & 54.99\tiny(2.48)          & 80.82\tiny(1.83)          & 62.61\tiny(0.75)          & 71.19\tiny(1.25)          & 90.70\tiny(0.43)          \\
		$\text{RC}_{\text{img}{1\mhyphen7}}$\tiny, $p$=0.5 & 98.83\tiny(0.07) & 86.08\tiny(0.27)          & 54.93\tiny(1.27)          & \textbf{81.58\tiny(0.74)} & \textbf{62.78\tiny(0.86)} & \textbf{71.34\tiny(0.61)} & \textbf{91.18\tiny(0.38)} \\
		$\text{RC}_{\text{img}{1\mhyphen9}}$\tiny, $p$=0.5 & 98.80\tiny(0.12) & 85.63\tiny(0.70)          & 52.82\tiny(2.01)          & 81.48\tiny(1.22)          & 62.55\tiny(0.74)          & 70.62\tiny(0.73)          & 90.79\tiny(0.48)              \\
		\bottomrule
	\end{tabular}
\end{table}
\begin{table}[htp]
	\small
	\centering
	\setlength{\tabcolsep}{3pt}
	\caption{Ablation study of consistency loss weight $\lambda$ on digits recognition benchmarks for $\text{RC}_{\text{mix}1\mhyphen7}$ and $\text{RC}_{\text{img}1\mhyphen7, p=0.5}$. DG-Avg is the average performance on MNIST-M, SVHN, SYNTH and USPS. Best results for each variant are \textbf{bold}.}
	\label{digits-consistency}
	\begin{tabular}{l|c|c|ccccc|c}
		\toprule
		& $\lambda$& MNIST-10k    & MNIST-M    & SVHN        & USPS        & SYNTH       & DG Avg         & MNIST-C     \\
		\midrule
		\multirow{6}{*}{$\text{RC}_{\text{mix}1\mhyphen7}$} & 20   & 98.90 \tiny(0.05) & 87.18 \tiny(0.81)          & \textbf{57.68 \tiny(1.64)} & \textbf{83.55 \tiny(0.83)} & 63.08 \tiny(0.50)          & 72.87 \tiny(0.47)          & 91.14 \tiny(0.53)          \\
		& 10   & 98.85 \tiny(0.04) & \textbf{87.76 \tiny(0.83)} & 57.52 \tiny(2.09)          & 83.36 \tiny(0.96)          & 62.88 \tiny(0.78)          & \textbf{72.88 \tiny(0.58)} & \textbf{91.62 \tiny(0.77)} \\
		& 5    & 98.94 \tiny(0.09) & 87.53 \tiny(0.51)          & 55.70 \tiny(2.22)          & 83.12 \tiny(1.08)          & 62.37 \tiny(0.98)          & 72.18 \tiny(1.04)          & 91.46 \tiny(0.50)          \\
		& 1    & 98.95 \tiny(0.05) & 86.77 \tiny(0.79)          & 56.00 \tiny(2.39)          & 83.13 \tiny(0.71)          & \textbf{63.18 \tiny(0.97)} & 72.27 \tiny(0.82)          & 91.15 \tiny(0.42)          \\
		& 0.1  & 98.84 \tiny(0.07) & 85.41 \tiny(1.02)          & 56.51 \tiny(1.58)          & 81.84 \tiny(1.14)          & 61.86 \tiny(1.44)          & 71.41 \tiny(0.98)          & 90.72 \tiny(0.60)          \\
		& 0    & 98.82 \tiny(0.06) & 84.91 \tiny(0.68)          & 55.61 \tiny(2.63)          & 82.09 \tiny(1.00)          & 62.15 \tiny(1.30)          & 71.19 \tiny(1.21)          & 90.30 \tiny(0.44)          \\
		
		\midrule
		\multirow{6}{*}{$\text{RC}_{\text{img}1\mhyphen7, p=0.5}$} & 20  & 98.79 \tiny(0.04) & 87.53 \tiny(0.79)          & 53.92 \tiny(1.59)          & 81.83 \tiny(0.70)          & 62.16 \tiny(0.37)          & 71.36 \tiny(0.49)          & \textbf{91.20 \tiny(0.53)}          \\
		& 10  & 98.86 \tiny(0.05) & 87.67 \tiny(0.37)          & 54.95 \tiny(1.90)          & 82.08 \tiny(1.46)          & 63.37 \tiny(1.58)          & 72.02 \tiny(1.15)          & 90.94 \tiny(0.51)          \\
		& 5   & 98.90 \tiny(0.04) & \textbf{87.77 \tiny(0.72)} & \textbf{55.00 \tiny(1.40)} & \textbf{82.10 \tiny(0.55)} & \textbf{63.58 \tiny(1.33)} & \textbf{72.11 \tiny(0.62)} & {90.83 \tiny(0.71)}          \\
		& 1   & 98.86 \tiny(0.04) & 86.74 \tiny(0.32)          & 53.26 \tiny(2.99)          & 81.51 \tiny(0.48)          & 62.00 \tiny(1.15)          & 70.88 \tiny(0.93)          & 91.11 \tiny(0.62)          \\
		& 0.1 & 98.85 \tiny(0.14) & 86.85 \tiny(0.31)          & 53.55 \tiny(3.63)          & 81.23 \tiny(1.02)          & 62.77 \tiny(0.80)          & 71.10 \tiny(1.31)          & 91.13 \tiny(0.69)          \\
		& 0   & 98.83 \tiny(0.07) & 86.08 \tiny(0.27)          & 54.93 \tiny(1.27)          & 81.58 \tiny(0.74)          & 62.78 \tiny(0.86)          & 71.34 \tiny(0.61)          & 91.18 \tiny(0.38)          \\
		
		\bottomrule
	\end{tabular}
\end{table}
\newpage

\section{More Examples of {\RandConv} Data Augmentation}
\label{examples}
We provide additional examples of {\RandConv} outputs for different convolution filter sizes in Fig.~\ref{fig:randconv_example_more} and for its mixing variants at scale $k=7$ with different mixing coefficients in Fig.~\ref{fig:randconv_mix_example_more}. We observe that {\RandConv} with different filter sizes retains shapes at different scales. The mixing strategy can continuously interpolate between the training domain and a randomly sampled domain.  


\begin{figure}[htp]
	\begin{center}
		\setlength{\tabcolsep}{0.01cm}
		\newcommand\cwidth{0.14\textwidth}
		\begin{adjustbox}{max width=\textwidth}
			\begin{tabular}{ccccccc}
				Input & $\alpha=0.9$ & $\alpha=0.7$ & $\alpha=0.5$ & $\alpha=0.3$ & $\alpha=0.1$ & $\alpha=0$ \\
				\includegraphics[width=\cwidth]{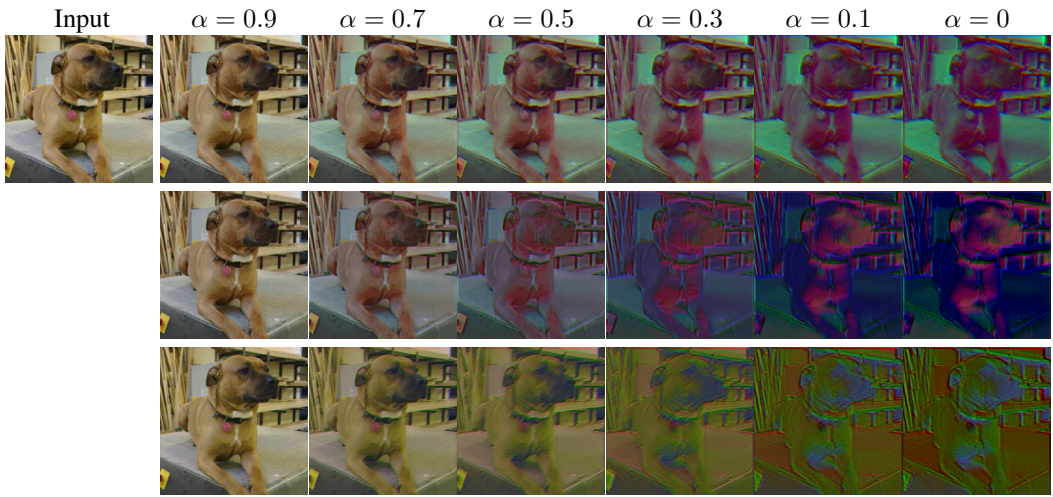}
				\forloop{sample_id}{1}{\value{sample_id} < 4}{ 
					&\includegraphics[width=\cwidth]{{Fig/examples/image2_kernel7_mix0.9_sample\arabic{sample_id}}.png} 
					&\includegraphics[width=\cwidth]{{Fig/examples/image2_kernel7_mix0.7_sample\arabic{sample_id}}.png} 
					&\includegraphics[width=\cwidth]{{Fig/examples/image2_kernel7_mix0.5_sample\arabic{sample_id}}.png} 
					&\includegraphics[width=\cwidth]{{Fig/examples/image2_kernel7_mix0.3_sample\arabic{sample_id}}.png} 
					&\includegraphics[width=\cwidth]{{Fig/examples/image2_kernel7_mix0.1_sample\arabic{sample_id}}.png} 
					&\includegraphics[width=\cwidth]{{Fig/examples/image2_kernel7_mix0_sample\arabic{sample_id}}.png} \\
				} 
			\end{tabular}
		\end{adjustbox}
	\end{center}
	\caption{\small Examples of the {\RandConv} mixing variant $\text{RC}_{\text{mix}7}$ on images of size $224^2$ with different mixing coefficients $\alpha$. When $\alpha=1$, the output is just the original image input;when $\alpha=0$, we use the output of the random convolution layer as the augmented image.}  
	\label{fig:randconv_mix_example_more}
\end{figure}

\begin{figure}[htp]
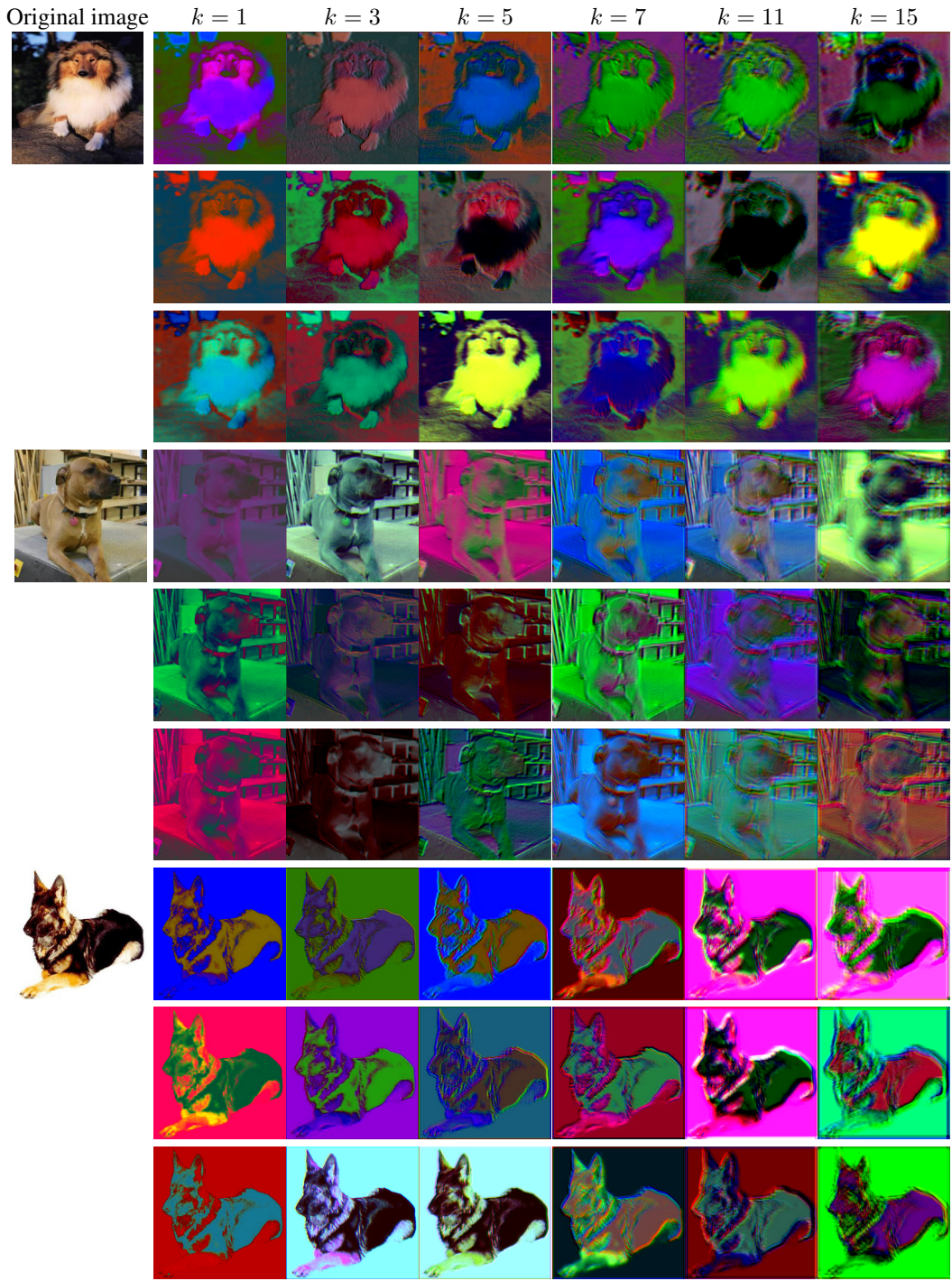

	\begin{center}
		\setlength{\tabcolsep}{0.01cm}
		\newcommand\cwidth{0.14\textwidth}
		\begin{adjustbox}{max width=\textwidth}
			\begin{tabular}{ccccccc}
				
				Original image & $k=1$ & $k=3$ & $k=5$ & $k=7$ & $k=11$ & $k=15$\\
				\forloop{imgnum}{1}{\value{imgnum} < 4}{
					\includegraphics[width=\cwidth]{Fig/examples/image\arabic{imgnum}.png} 
					\forloop{sample_id}{0}{\value{sample_id} < 3}{
						& \includegraphics[width=\cwidth]{Fig/examples/image\arabic{imgnum}_kernel1_sample\arabic{sample_id}.png}
						& \includegraphics[width=\cwidth]{Fig/examples/image\arabic{imgnum}_kernel3_sample\arabic{sample_id}.png} 
						& \includegraphics[width=\cwidth]{Fig/examples/image\arabic{imgnum}_kernel5_sample\arabic{sample_id}.png}
						& \includegraphics[width=\cwidth]{Fig/examples/image\arabic{imgnum}_kernel7_sample\arabic{sample_id}.png} 
						& \includegraphics[width=\cwidth]{Fig/examples/image\arabic{imgnum}_kernel11_sample\arabic{sample_id}.png}
						& \includegraphics[width=\cwidth]{Fig/examples/image\arabic{imgnum}_kernel15_sample\arabic{sample_id}.png}
						\\
					}  
				}\\
				
			\end{tabular}
		\end{adjustbox}
	\end{center}
	\caption{\small {\RandConv} data augmentation examples on images of size $224^2$. First column is the input image; following columns are convolution results using random filters of different sizes $k$. We can see that the smaller filter sizes help maintain the finer shapes.}
	\label{fig:randconv_example_more}
\end{figure}

\end{document}